\documentclass{article} 
\usepackage{preprint,times}

\usepackage{mathtools, amsmath, amssymb, xcolor, amsthm, enumerate}
\usepackage{nicefrac}
\usepackage{enumitem}
\usepackage{booktabs} 
\usepackage{bbm}
\usepackage{algorithm}
\usepackage{algpseudocode}
\usepackage{tabularx}
\usepackage{titlesec}
\usepackage{subcaption}
\usepackage{tikz}
\titleformat{\paragraph}[runin]{\normalfont\bfseries}{\theparagraph}{1em}{}
\titlespacing{\paragraph}{0pt}{1pt}{1em}

\usepackage{amsmath,amsfonts,bm}

\def\1{\bm{1}}

\DeclareMathAlphabet{\mathsfit}{\encodingdefault}{\sfdefault}{m}{sl}
\SetMathAlphabet{\mathsfit}{bold}{\encodingdefault}{\sfdefault}{bx}{n}

\newcommand{\R}{\mathbb{R}}

\DeclareMathOperator{\sign}{sign}

\usepackage{hyperref}
\usepackage{url}

\newtheorem{theorem}{Theorem}
\newtheorem{lemma}[theorem]{Lemma} 
\newtheorem{proposition}[theorem]{Proposition}

\newtheorem{corollary}[theorem]{Corollary}

\newcommand{\N}{{\mathbb N}}
\renewcommand{\d}{\, \mathrm{d}}
\newcommand{\e}{\mathrm{e}}

\newcounter{algline}
\newcommand{\algrow}{\stepcounter{algline}\arabic{algline} & }

\title{Quasi Linear Kernel Attention with Infinite Capacity}

\author{Nicolaj Rux$^{1}$, Johannes Hertrich$^{2}$ \& Sebastian Neumayer$^{1}$ \\
$^{1}$Faculty of Mathematics, Chemnitz University of Technology, 09126 Chemnitz, Germany \\
$^{2}$Institute of Computer Science, University of Göttingen, 37073 Göttingen, Germany \\
\texttt{\{nicolaj.rux,sebastian.neumayer\}@mathematik.tu-chemnitz.de} \\
\texttt{johannes.hertrich@uni-goettingen.de}
}

\iclrfinalcopy 
\begin{document}

\maketitle

\begin{abstract}
The evaluation cost of transformers with softmax attention scales quadratically with sequence length. Kernel attention addresses this by replacing softmax with a more general kernel function. In this paper, we aim to identify kernels that retain the expressivity of attention while enabling quasi linear computation.
To quantify expressivity, we introduce a capacity for each kernel, measuring the maximum sequence length for which the attention matrix can approximate the identity. A higher capacity thus indicates greater expressivity. We show that expressive kernels like softmax, Gauss, and Laplace have infinite capacity. In contrast, common quasi linear kernels, such as those derived from finite dimensional feature maps, exhibit finite capacity.
As a solution, we propose additive kernels constructed from univariate spline and polynomial exponential kernels. We prove that these maintain infinite capacity while allowing quasi linear computation via sorting. Finally, we implement additive sorting kernels efficiently and benchmark them against modern softmax backends, demonstrating advantages for long sequences.
\end{abstract}

\section{Introduction}
\label{sec:intorduction}
Transformers \citep{VSPUJGKP2017AttentionIsAllYouNeed} are the dominant architecture in language processing \citep{Brown2020} and computer vision \citep{DBKWZUDMHGUH2021ViT}.
Their central building block is attention, which computes for \emph{query vectors} ${\bm q}=(q_m)_{m=1}^M\subseteq\R^D$, \emph{key vectors} ${\bm k}=(k_n)_{n=1}^N\subseteq\R^D$, and \emph{value vectors} ${\bm v}=(v_n)_{n=1}^N\subseteq\R^C$ the $M$ \emph{attention} vectors $(y_m)_{m=1}^M$ as
\begin{equation}\label{eq:kernel_attention}
   y_m={\sum_{n=1}^{N}\Phi(q_m,k_n)v_n}\Bigr/{\sum_{n=1}^{N}\Phi(q_m,k_n)},
    \qquad m=1,\dots,M.
\end{equation}
Most implementations use the \emph{softmax} kernel $\Phi(q,k) = \e^{\tau q^\top k}$ with inverse temperature $\tau>0$, which acts as similarity score between the vectors $q$ and $k$.
For $M=N$, the computational effort scales quadratically with the \emph{sequence length} $N$.
This has become the main computational bottleneck in practice.
Current models are trained and deployed with context windows of $128$k tokens and beyond \citep{llama3herd2024}, which involves substantial engineering effort \citep{dao2022flash, dao2024flash2}.
There have been several attempts to escape the quadratic complexity through approximations and windowed attention patterns; see \cite{tdbm2022EffTrans} for an overview.

In this paper, we focus on \emph{exact} subquadratic attention, where the structure of the kernel $\Phi$ allows the computation of \eqref{eq:kernel_attention} in quasi linear time \citep{TBYMS2019tsai}.
Within this class, \cite{KVPF_Trans_are_RNNs} achieved linear complexity by choosing kernels that factor over a \emph{finite feature map} (FFM) $\varphi\colon \R^D\to \R^L$, i.e., $\Phi(q,k)=\varphi(q)^\top \varphi(k)$.
This construction yields symmetric positive definite (spd) kernels, which can be evaluated in linear time.
However, for small $L$, these approaches have been shown to admit only limited expressivity compared to softmax attention, see \cite{sis2021FastWeight}.
At the same time, the cost of evaluating the finite feature map increases linearly with $L$.

On the other hand, there exist several univariate kernels (namely $D=1$), for which the corresponding kernel sums can be evaluated exactly via quasi linear sorting algorithms.
Examples are polynomial exponential kernels like the Laplace kernel \citep{H2021FastExactEval} and piecewise linear kernels \citep{Johannes2024KernelSum, generative_sliced, VB2025slicedReLU}. However, most of these algorithms do not generalize to higher dimensions.
An exception is the Laplace kernel, for which the sums can be evaluated in quasi linear time for arbitrary $D$, see \cite{LW2021FastLaplaceD}, but the complexity depends exponentially on $D$, making the approach practically infeasible for $D>6$.

\paragraph{Contribution}

Our goal is to identify kernels $\Phi$ such that the attention \eqref{eq:kernel_attention} is both expressive and computable in quasi linear time.
To this end, we introduce the \emph{capacity} of a kernel.
Its definition is inspired by the associative recall experiment of \cite{sis2021FastWeight} and measures up to which sequence length the identity can be approximated by the attention algorithm. 
Intuitively, this corresponds to the number of tokens that can be distinguished. 
Consequently, a higher capacity indicates higher expressivity.
We prove that common choices have infinite capacity, including the softmax, Gauss, and Laplace kernels.
This resembles the observation that their associated attention is expressive in practice.
For kernels based on an FFM, we prove that the capacity is upper-bounded by the number $L$ of considered features, reflecting their limited expressivity.

In order to achieve quasi linear attention with infinite capacity, we consider additive kernels of the form $\Phi(q,k)=\phi(q_1,k_1)+\ldots +\phi(q_D,k_D)$, where $\phi$ is a univariate kernel with quasi linear attention.
For the univariate Laplace kernel $\phi(s,t)=\exp(-|s-t|)$ and the bump kernel $\phi(s,t)=\max(0,1-|s-t|)$, we prove that the capacity of both $\phi$ and $\Phi$ is infinite, so that we obtain quasi linear attention with infinite capacity.
We verify our findings in numerical experiments and provide an efficient CUDA implementation.
For sequence lengths $N \gtrsim 2{,}048$, our implementation is faster than conventional single-precision PyTorch attention, which offers a comparable level of numerical precision and code optimization.
Compared to the highly optimized half-precision FlashAttention, we reach the break-even point for $N \gtrsim 20{,}000$.
Thanks to the quasi linear complexity of our algorithm, the advantage increases rapidly with the sequence length.

\paragraph{Related work}
Several works reduce the quadratic cost of attention by approximation, either through sparse and windowed attention patterns \citep{KKL2020Reformer, BPC2020Longformer, ZGDAAOPRWY2020BigBird}, or through low-rank approximations of the attention matrix \citep{WLKFM2020Linformer, XZCTFLS2021Nystromformer}.
For an extensive overview, see \cite{tdbm2022EffTrans}.
The FFM idea was refined by \cite{performer2021, PPYSSK2021RFA, QSDLWLYKZ2022cosFormer} with feature maps tailored to the softmax kernel and by \cite{sis2021FastWeight}, who increased the \emph{feature dimension} $L$ after recognizing that it caps the attention capacity.
Lower bounds on the recurrent state size required for multi-query associative
recall were proven by \cite{ABTPZRR2024zoology}.
Building on this, \cite{Based2024} combine linear attention with sliding windows, while gated linear attention \citep{GLA2024} and DeltaNet \citep{DeltaNet2024} use data-dependent gating and delta-rule updates of the recurrent state, respectively, to better use a fixed-size state.
These models are recurrent and go beyond the normalized kernel sums considered in \eqref{eq:kernel_attention}.
Alternative notions of expressivity for attention mechanisms were considered in \cite{Yun2020Are, furuya2025transformers}.
Our proposed bump and Laplace variants have been previously considered in different contexts:
\cite{flwztxmz2026laplaceformer} approximate the Laplace kernel attention via a finite feature map/Nyström approach, and \cite{VB2025slicedReLU} combine the univariate bump kernel with a linear or MLP layer before each attention module to project onto a single dimension.
Finally, the additive kernels are related to kernel slicing, see also Appendix~\ref{app:sec:slicing}.

\section{Capacity of Kernel Attention}
\label{sec:kme}

In order to determine which kernels $\Phi$ lead to an expressive attention mechanism, we assign to each $\Phi$ a capacity $\mathrm{Cap}(\Phi)$.
Intuitively, this measures the maximal number of tokens that can be distinguished by the attention mechanism, which corresponds to the largest sequence length such that the identity can be approximated.
Consequently, a larger capacity corresponds to more expressive attention.
In this section, we formally define $\mathrm{Cap}(\Phi)$ and derive sufficient conditions such that $\mathrm{Cap}(\Phi)=\infty$.
These conditions apply in particular to the Gauss, softmax, and Laplace kernels, which are empirically known to correspond to expressive attention mechanisms.

\paragraph{Definition of Capacity}
For a kernel $\Phi$, keys $\bm k=(k_n)_{n=1}^N$ and queries \smash{$\bm q=(q_m)_{m=1}^M$}, we define the \emph{Gram matrix} $G(\bm q,\bm k)=(G_{m,n})_{m,n=1}^{M,N}$ and the \emph{attention matrix} $A(\bm q,\bm k)=(A_{m,n})_{m,n=1}^{M,N}$ by
\begin{equation}\label{eq:gram_attention}
    G_{m,n}= \Phi(q_m, k_n) \qquad \text{ and }\qquad  A_{m,n} = {G_{m,n}}/{\textstyle \sum_{l=1}^N G_{m,l}}.
\end{equation}
For $\bm v = (v_n)_{n=1}^N$ in $\R^C$, the attention mechanism can be rewritten as the matrix-vector multiplication $y_{c} = A(\bm q,\bm k) v_{c}$ for $c=1,\dots,C$.
We denote the \emph{capacity} $\mathrm{Cap}(\Phi)$ of $\Phi$ as the largest sequence length $N$ such that there exist keys $\bm k$ and queries $\bm q$ for which $A(\bm q,\bm k)\approx \mathrm{Id}_N$, where $\mathrm{Id}_N$ is the $N\times N$ identity matrix.
Formally, we define
\begin{equation}
    \mathrm{Cap}(\Phi)=\sup\left\{N\in\N:\inf_{\bm q,\bm k\in (\R^D)^N}\|A(\bm q,\bm k)-\mathrm{Id}_N\|_F^2=0\right\}.
\end{equation}
In other words, if the sequence length $N$ is larger than $\mathrm{Cap}(\Phi)$, the attention mechanism is no longer able to represent the identity, which indicates a lack of expressiveness.

\paragraph{Kernels with Infinite Capacity}
A kernel $\Phi$ is stationary if $\Phi(q,k)=F(q-k)$ for some function $F$.
If $F$ is continuous and fulfills $F(x)\to 0$ as $\|x\|_2\to\infty$, we call $\Phi$ a stationary $\mathcal C_0$ kernel. 
\begin{theorem}\label{thm:inf_cap}
Let $\Phi$ be a stationary $\mathcal C_0$ kernel with $F(0)>0$, then it holds $\mathrm{Cap}(\Phi)=\infty$.
\end{theorem}
The theorem applies in particular to the Gauss kernel \smash{$\Phi(q,k)=\exp(-\frac12\|q-k\|_2^2)$} and many others including Laplace and Mat\'ern kernels.
Even though the softmax kernel $\Phi(q,k)=\exp(q^\top k)$ is not stationary, we can exploit its close relation to the Gauss kernel to show that $\mathrm{Cap}(\Phi)=\infty$.
\begin{proposition}\label{prop:softmax_capacity}
For $D\ge 2$ and $\Phi(q,k)=\exp(q^\top k)$, we have $\mathrm{Cap}(\Phi)=\infty$.
\end{proposition}

\paragraph{Upper Bounding the Capacity}
Next, we derive a criterion which upper bounds the capacity of an spd kernel $\Phi$ by the dimension of the corresponding reproducing kernel Hilbert space (RKHS).
To define the RKHS, we consider the space
$
\mathcal H_\Phi^0\coloneqq \mathrm{span}\{\Phi(\cdot, k):k\in\R^D\}
$
with the bilinear form
\begin{equation}
\left\langle \sum_{m=1}^M a_m \Phi(\cdot, q_m),\sum_{n=1}^N b_n \Phi(\cdot, k_n)\right\rangle_{\mathcal H_\Phi} = \sum_{m=1}^{M}\sum_{n=1}^N a_m b_n \Phi(q_m,k_n).
\end{equation}
Then, the RKHS $\mathcal H_\Phi$ is defined as the completion of $\mathcal H_\Phi^0$ with respect to the norm induced by $\langle\cdot,\cdot\rangle_{\mathcal H_\Phi}$. 
The columns of the Gram matrix $G(\bm q,\bm k)$ can be seen as functions $\Phi(\cdot, k_n)\in \mathcal H_\Phi^0$ evaluated at queries $\bm q$. 
Therefore, the rank of $G(\bm q,\bm k)$ can be bounded via
\begin{equation}
    \mathrm{rank}(G(\bm q,\bm k))\leq \dim(\mathrm{span}\{\Phi(\cdot,k_1),\ldots, \Phi(\cdot,k_N)\})\leq \dim(\mathcal H_\Phi).
\end{equation}
As the attention matrix $A(\bm q,\bm k)$ arises from $G(\bm q,\bm k)$ by rescaling the rows, we also find that \smash{$\mathrm{rank}(A(\bm q,\bm k))=\mathrm{rank}(G(\bm q,\bm k))\leq \mathrm{dim}(\mathcal H_\Phi)$}.
In particular, we have for any $\bm q$ and $\bm k$ by the Eckart-Young theorem \citep{EY1936low_rank} that 
\begin{equation}
    \inf_{\bm q,\bm k\in (\R^D)^N}\|A(\bm q,\bm k)-\mathrm{Id}_N\|_F^2\geq N-\mathrm{dim}(\mathcal H_\Phi).
\end{equation}
Thus, the left hand side can only be zero if $N\leq \mathrm{dim}(\mathcal H_\Phi)$ such that $\mathrm{Cap}(\Phi)\leq \mathrm{dim}(\mathcal H_\Phi)$.
We summarize this result in the following theorem.
\begin{theorem}\label{thm:dim_cap}
Let $\Phi$ be an spd kernel, then it holds $\mathrm{Cap}(\Phi)\leq \mathrm{dim}(\mathcal H_\Phi)$.
\end{theorem}
The reverse statement is not true.
More precisely, we show in Proposition~\ref{prop:riesz_capacity} that there exist kernels $\Phi$ such that $\mathrm{dim}(\mathcal H_\Phi)=\infty$, but $\mathrm{Cap}(\Phi)<\infty$.

\section{Capacity of quasi linear Attention}
\label{sec:capacity}

In Theorem~\ref{thm:inf_cap}, we have derived a large class of kernels which admits infinite capacity, indicating a high expressivity.
However, the computation of the attention mechanism requires generally $\mathcal O(NM)$ operations, which is a major computational bottleneck.
Therefore, we now focus on kernels for which the attention mechanism can be computed exactly in quasi linear time, i.e., in $\mathcal O((N{+}M)\log^\nu(N{+}M))$ operations for some $\nu>0$.
We will mainly analyze the complexity of computing the kernel sums
\begin{equation}\label{eq:kernel_sum}
    z_m=\sum_{n=1}^{N}\Phi(q_m, k_n)v_n, \qquad m=1,\dots,M,
\end{equation}
for queries $\bm q=(q_m)_{m=1}^M$, keys $\bm k=(k_n)_{n=1}^N$ and values $\bm v=(v_n)_{n=1}^N$.
In fact, from a computational viewpoint, considering the kernel sums is equivalent to considering the attention mechanism, as the normalization can be computed in $\mathcal O(MC)$ operations.
In practice, the time of computing this normalization is dominated by the $\mathcal O((N{+}M)(D{+}C))$ cost of reading the data.

\paragraph{Laplace}
\citet{LW2021FastLaplaceD} propose an algorithm based on \citep{Bently1980} with quasi linear complexity to compute kernel sums for the Laplace kernel $\Phi(q,k)=\exp(-\|q-k\|_1)$, which fulfills $\mathrm{Cap}(\Phi)=\infty$ by Theorem~\ref{thm:inf_cap}.
However, the algorithm relies on a decomposition of the Laplace kernel sum into $2^D$ generalized empirical cumulative distribution functions. 
In particular, the complexity depends exponentially on the dimension.
Thus, for an exemplary head dimension of $D=64$, the quasi linear implementation is only expected to be faster than brute-force for sequence lengths larger than $2^{64} > 10^{19}$, which makes it a theoretical advance rather than a practical method.

\subsection{Finite Feature Maps}

Following the idea of random Fourier features \citep{RR2007RandomFeatures}, several works \citep{PPYSSK2021RFA,performer2021,KVPF_Trans_are_RNNs, sis2021FastWeight} proposed linear attention algorithms by using finite feature maps (FFMs) $\varphi\colon\R^D\to\R^L$.
The main idea is to choose $\Phi(q,k)=\varphi(q)^\top\varphi(k)$, for which the kernel sum can be computed as
\begin{equation}
 z_m=\sum_{n=1}^N \Phi(q_m,k_n)v_n
  = \sum_{n=1}^N \varphi(q_m)^\top\varphi(k_n)v_n
  = \Big(\sum_{n=1}^N v_n\varphi(k_n)^\top\Big)\varphi(q_m).
\end{equation}
The inner sum $\sum_{n=1}^N v_n\varphi(k_n)^\top$ no longer depends on $m$ such that all $z_m$ together can be computed in linear complexity $\mathcal O(LC(N{+}M))$.
\citet{sis2021FastWeight} uses the feature dimension as a measure of capacity.
This is reasonable because
\begin{equation}
H_\Phi^0= \operatorname{span} \{ \Phi(\cdot, k)\mid k\in \R^D\}
=  \operatorname{span}\{\varphi(\cdot)^\top \varphi(k)\mid k\in \R^D\}
\subseteq   \operatorname{span}\{\varphi_1,\ldots, \varphi_L\}.
\end{equation}
Since the completion of a finite dimensional space is the space itself, we also obtain $\mathcal H_\Phi\subseteq \mathrm{span}\{\varphi_1,...,\varphi_L\}$, which has at most dimension $L$. 
Corollary~\ref{corr:ffm_finite_cap} gives the connection to \cite{sis2021FastWeight} as an immediate consequence of Theorem~\ref{thm:dim_cap}.
\begin{corollary}\label{corr:ffm_finite_cap}
Let $\Phi(q,k)=\varphi(q)^\top\varphi(k)$ with FFM $\varphi\colon\R^D\to\R^L$, then $\mathrm{Cap}(\Phi)\leq \mathrm{dim}(\mathcal H_\Phi)\leq L$.
\end{corollary}
By Corollary~\ref{corr:ffm_finite_cap} the capacity is bounded through the feature dimension $L$.
At the same time, the computational complexity scales linearly with $L$ such that FFMs with more features are slower to evaluate.
The next proposition verifies that the assumptions of Theorem~\ref{thm:inf_cap} are indeed violated for FFMs in the sense that most FFMs are not stationary and the only stationary FFMs are not $\mathcal C_0$.
\begin{proposition}\label{prop:stationary_FFM}
Let $\Phi(q,k)=\varphi(q)^\top\varphi(k)$ be a continuous and stationary kernel. Then, it holds \smash{$\Phi(q,k)=\sum_{l=1}^L \alpha_l \cos(\omega_l^\top (q-k))$}, where  $\alpha_l\geq0$ and $\omega_l\in \R^D$ for $l=1,\ldots, L$.
In particular, we have that $\Phi$ is a stationary $\mathcal C_0$ kernel if and only if $\Phi=0$.
\end{proposition}

\subsection{Quasi Linear Attention with infinite capacity in One dimension}

There exist several univariate kernels $\phi\colon\R\times\R\to \R$, for which kernel sums can be computed in quasi linear time via sorting.
This involves the Laplace kernel $\phi(s,t)=\exp(-|s-t|)$ \citep{H2021FastExactEval, LW2021FastLaplaceD},
the distance kernel $\phi(s,t)=|s-t|$ \citep{Johannes2024KernelSum, generative_sliced, TSGSW2011dithering} and polynomial exponential kernels $\phi(s,t)=p(|s-t|)\exp(-|s-t|)$ for some polynomial $p$ \citep{H2021FastExactEval}. 
Using that any continuous piecewise linear function $f\colon\R\to\R$ can be represented as the composition of constants, linear terms and absolute values \citep{SGRB2014}, the sorting algorithm for the distance kernel can be deployed for any kernel $\phi(s,t)=f(s-t)$. 
\begin{proposition}\label{prop:spline_is_quasi}
Let $f\colon \R\to\R$ be piecewise linear with finitely many nodes, then $\phi(s,t)=f(s-t)$ is quasi linear.
\end{proposition}
This includes the bump kernel $f(u)=\max(0,1-|u|)$ with  alternative expressivity results by \citet[Thm.\ 3.2]{VB2025slicedReLU}.
Both bump and Laplace have infinite capacity, see Theorem~\ref{thm:inf_cap}.
\begin{corollary}\label{cor:LaplaceBump}
Both the bump kernel $\phi(s,t)=\max\{0, 1-|s-t|\}$ and the Laplace kernel $\phi(s,t)=\exp(-|s-t|)$ are quasi linear and have infinite capacity.
\end{corollary}
On the other hand, the Riesz kernel $r(s,t)=|s|+|t|-|s-t|$ is an example for a kernel with infinite dimensional RKHS and finite capacity.
Since the RKHS of $r$ can be identified with the infinite dimensional homogeneous Beppo Levi space \cite[Prop.~10.39]{Wendland2004}, its attention matrix has no bound on its rank.
Yet, its global structure does not allow to approximate the identity.
\begin{proposition}\label{prop:riesz_capacity}
For the Riesz kernel $\phi(s,t)=|s|+|t|-|s-t|+\varepsilon$ with $\varepsilon>0$, we have $\mathrm{Cap}(\phi)= 2$.
\end{proposition}

\paragraph{Weighted Absolute Value Sums}
For the Riesz and spline kernels, the difficult part is to compute \smash{$\sum_{n=1}^{N} |s_m-t_n|v_n$}, where $s_m\in \R$ and $t_n\in \R$.
This sum can be computed in $\mathcal O((N{+}M)(\log_2 N{+}C))$ by Algorithm~\ref{alg:weighted_abs_sum_forward}.
A similar algorithm can be applied for the one dimensional Laplace kernel.
\begin{algorithm}[t]
\caption{Weighted absolute value sum}
\label{alg:weighted_abs_sum_forward}

\setcounter{algline}{0}
\begin{tabularx}{\linewidth}{@{}r X l l@{}}
\textbf{\#} & \textbf{Step} & \textbf{Work}& \textbf{Memory} \\
\hline

\algrow \textbf{Input}
$\bm s\in \R^M$, $\bm t\in \R^N$ and $\bm v\in \R^{N\times C}$
& -
& $(C{+}1)N{+}M$
\\
\algrow \textbf{Output}
$z_m \coloneqq \sum_{n=1}^{N} |s_m-t_n|v_n\in \R^C \quad m=1,\ldots, M$
& -
& $MC$
\\
\algrow
$\sigma\coloneqq \texttt{argsort}(\bm t)$ s.t. $t_{\sigma(1)}\leq \ldots \leq t_{\sigma(N)}$ 
& $N\log_2 N$
& $N$
\\
\algrow
$p_m \coloneqq \max \bigl(\{n\colon t_{\sigma(n)}\leq s_m\}\cup\{0\}\bigr)$ for $m=1,\ldots, M$ 
& $M\log_2 N$
& $M$
\\
\algrow $\textrm{pref\_t}_n\coloneqq \sum_{j=1}^{n} t_{\sigma(j)}v_{\sigma(j)}\in \R^C$ for $n=0,\ldots, N$ 
& $(N{+}1)C$
& $(N{+}1)C$
\\
\algrow
$\textrm{pref\_v}_n\coloneqq \sum_{j=1}^{n} v_{\sigma(j)}\in \R^C$ for $n=0,\ldots, N$ 
& $(N{+}1)C$
& $(N{+}1)C$
\\
\algrow
$a_m \coloneqq 2\,\textrm{pref\_v}_{p_m}-\textrm{pref\_v}_N\in \R^C$ for $m=1,\ldots,M$
& $MC$
& $MC$
\\
\algrow
$z_m \coloneqq \textrm{pref\_t}_N - 2\,\textrm{pref\_t}_{p_m} +s_m a_m\in \R^C$ for $m=1,\ldots,M$
& $MC$
& $MC$
\\
\algrow \textbf{return}
$\bm z$
&
&
\\
\end{tabularx}
\end{algorithm}
The core idea is that after sorting, the keys below and above $s_m$ separate:
\begin{align}
&
\sum_{n=1}^{N} |s_m-t_n|v_n
= \sum_{n=1}^{N} |s_m-t_{\sigma(n)}|v_{\sigma(n)}\notag\\
=& \sum_{n=1}^{p_m} (s_m-t_{\sigma(n)})v_{\sigma(n)} +  \sum_{n=p_m+1}^{N} (t_{\sigma(n)}-s_m)v_{\sigma(n)}\notag\\
=& s_m\sum_{n=1}^{p_m} v_{\sigma(n)} - \sum_{n=1}^{p_m}t_{\sigma(n)}v_{\sigma(n)} +
\sum_{n=p_m+1}^{N} t_{\sigma(n)} v_{\sigma(n)} - s_m \sum_{n=p_m+1}^{N}v_{\sigma(n)}\notag\\
=& s_m(2\, \text{pref\_v}_{p_m} - \text{pref\_v}_{N}) +  \text{pref\_t}_{N} - 2\, \text{pref\_t}_{p_m}
= z_m. \label{eq:prefix_sums}
\end{align}
The last line \eqref{eq:prefix_sums} depends on the prefix sums over $\bm v_{\sigma}$ and $\bm t_\sigma \bm v_\sigma$, which can be computed with $\mathcal O(NC)$ work.
Once the prefix sums have been computed, the entry at position $p_m$ needs to be extracted, resulting in another $\mathcal O(MC)$ operations.
The total work of the one dimensional absolute value sum reduces to $\mathcal O((N{+}M)(\log_2 (N){+}C))$ operations.
In short, we pay \smash{$\mathcal O((N{+}M)\log_2 N)$} to then be able to compute the rest in $\mathcal O(C(N{+}M))$ instead of $\mathcal O(CNM)$.

Appendix~\ref{app:sec:implementation} explains how we implement Algorithm~\ref{alg:weighted_abs_sum_forward} efficiently on the GPU.
For the multivariate case, a naive PyTorch implementation needs to build huge temporary tensors in steps $5$ to $8$. Meanwhile, we compute steps $5$ to $8$ in a single CUDA kernel and allocate the intermediate tensors \text{pref\_v} and \text{pref\_t} on the fly.
This ensures a low memory profile with good parallelism.
\subsection{Additive Kernels for Quasi Linear Attention with Infinite Capacity}

In order to carry over the one dimensional algorithms from the previous subsection to higher dimensions, we propose to use additive kernels of the form
\begin{equation}\label{eq:additive_kernel}
\Phi\colon\R^D\times\R^D\to\R,\qquad \Phi(q,k)=\sum_{d=1}^D \phi(q_d,k_d)
\end{equation}
for some univariate kernel $\phi\colon\R\times\R\to\R$.
Similar kernels were considered in \cite{QMC_slicing} as quasi-Monte Carlo design for the kernel slicing algorithm; more details are given in Appendix~\ref{app:sec:slicing}.
By definition, the kernel sums for such kernels can be computed as
\begin{equation}
    z_m=\sum_{n=1}^N \Phi(q_m,k_n)v_n=\sum_{d=1}^D\sum_{n=1}^N \phi(q_{m,d},k_{n,d})v_n,
\end{equation}
which has quasi linear complexity if and only if the kernel sum with respect to the univariate kernel $\phi$ has quasi linear complexity.
Even though additive kernels $\Phi$ will never be $\mathcal C_0$ (unless $\Phi=0$), we prove in the next theorem that $\mathrm{Cap}(\phi)=\infty$ implies that $\mathrm{Cap}(\Phi)=\infty$.
\begin{theorem}\label{thm:additive_capacity}
Let $\Phi$ be an additive kernel corresponding to some $\phi$.
Then, we have that $\mathrm{Cap}(\Phi)\geq \mathrm{Cap}(\phi)$.
Moreover, the properties spd, quasi linear and stationary carry over from  $\phi$ to $\Phi$. For the Riesz kernel $\phi=r_\varepsilon$, we have $\mathrm{Cap}(\Phi)\ge 2D$.
\end{theorem}
We now combine Corollary~\ref{cor:LaplaceBump} with the arguments in this section to get our main result.
\begin{corollary}\label{corr:add_bump_capacity}
The additive- Laplace and bump kernel defined by $\Phi(q,k)=\sum_{d=1}^D\phi(q_d,k_d)$ with $\phi(s,t)=\exp(-|s-t|)$ and $\phi(s,t)=\max(0,1-|s-t|)$ fulfill $\mathrm{Cap}(\Phi)=\infty$ and corresponding kernel sums can be computed in quasi linear complexity.
\end{corollary}

\section{Implementation and Experiments}
\label{sec:numerics}

Next, we investigate our findings numerically by three experiments.
The first one in Section~\ref{subsec:associative_recall} aims to verify theoretical claims from the previous sections.
To this end, we perform an associative-recall experiment, which numerically computes the capacity of the considered kernels.
The second experiment in Section~\ref{subsec:num_runtime} investigates the computation time.
To complement theoretical complexity analysis of the proposed quasi linear kernels, we investigate the ``break-even'' point.
This describes the minimal sequence length $N$ such that our attention algorithm is faster than softmax attention.
In this part, we also comment on possible implementations of attention. 
Finally, the third experiment in Section~\ref{subsec:num_text} evaluates the performance of the kernels based on the text embedding transformer nomic-embed-text-v1 by \cite{NMDM2024nomic}. We include a similar experiment for a vision transformer in Appendix~\ref{app:sec:vision_hyper}.
Table~\ref{tab:kernel_overview} gives an overview of the kernels evaluated in this section.

All experiments use PyTorch 2.13.0+cu129 and KeOps 2.3. While training runs on a multi GPU cluster, all time measurements (Table~\ref{tab:fwd_runtimes} \& \ref{tab:fwdbwd_runtimes} and Figure~\ref{fig:pareto}) were taken on a single NVIDIA GeForce RTX 5090. The code corresponding to our experiments is available on GitHub\footnote{    \url{https://github.com/Nicolaj-Rux/quasi-linear-kernel-attention}}.
\begin{table}
    \centering
    \caption{For additive kernels, we state the univariate $\phi(s,t)$.
    \texttt{dpfp} is described in Appendix~\ref{app:sec:FFM} and the remaining kernels are defined for $q,k\in \R^D$.
    All kernels except \texttt{tri} are nonnegative.
    \texttt{riesz} and \texttt{add\_riesz} use $\varepsilon=10^{-3}$ and are \emph{partially} stationary, i.e., stationary up to terms depending only on $q$ or only on $k$.
    The last column is the best known kernel sum complexity in $N{+}M$.
    }
    \scriptsize 
    \begin{tabular}{rlrrrlr}
        \toprule 
         Name                    & Formula                                      &    dim   & stationary & additive & capacity  & complexity   \\
         \midrule 
         \texttt{softmax}        & $\e^{q^\top k}$                              & $\infty$ & no         & no       & $\infty$  & quadratic    \\
         \texttt{gauss}          & \smash{$\e^{-\frac12 \|q-k\|_2^2}$}          & $\infty$ & yes        & no       & $\infty$  & quadratic    \\
         \texttt{laplace}        & \smash{$\e^{-\|q-k\|_1}$}                    & $\infty$ & yes        & no       & $\infty$  & quasi linear \\
         \texttt{riesz}          & $\|q\|_2+\|k\|_2-\|q-k\|_2+\varepsilon$      & $\infty$ & partially  & no       & $-$       & quadratic    \\
         \texttt{add\_riesz}     & $|s|+|t|-|s-t|+\varepsilon$                  & $\infty$ & partially  & yes      & $\geq 2D$       & quasi linear \\
         \texttt{add\_laplace}   & \smash{$\e^{-|s-t|}$}                        & $\infty$ & yes        & yes      & $\infty$  & quasi linear \\
         \texttt{add\_bump}      & $\max\{0, 1-|s-t|\}$                         & $\infty$ & yes        & yes      & $\infty$  & quasi linear \\
         \texttt{tri}            & $\cos(s-t)$                                  & $2D$     & yes        & yes      & $\leq 2D$     & linear       \\
         \texttt{relu}           & $\mathrm{ReLU}(s)\mathrm{ReLU}(t)$           & $D$      & no         & yes      & $\leq D$      & linear       \\
         \texttt{elu}            & $(\mathrm{elu}+1)(s)(\mathrm{elu}+1)(t)$     & $D$      & no         & yes      & $\leq D$      & linear       \\
         \texttt{dpfp}           & Appendix~\ref{app:sec:FFM}               & $6D$     & no         & no       & $\leq 6D$     & linear       \\
         \bottomrule
    \end{tabular}
    \label{tab:kernel_overview}
\end{table}

\paragraph{Bandwidth}
The bandwidth is set via $\Phi_\tau(q,k)\coloneqq \Phi(\nicefrac{q}{\tau}, \nicefrac{k}{\tau})$, with $\tau=D^{\nicefrac{1}{4}}$ as default for \texttt{softmax}, \texttt{gauss} and \texttt{tri}. Since both text and vision transformer use head dimension $D=64$, we use $\tau\approx2.828$.
We also tested other bandwidths but found $\tau=D^{\nicefrac{1}{4}}$ to be near optimal in both the text and vision transformer. 
Similarly, $\tau=1.5$ for \texttt{add\_bump}, $\tau=0.5$ for \texttt{add\_laplace} and $\tau = 6$ for \texttt{laplace} work well on both tasks.
For $\varepsilon=0$, the kernels \texttt{riesz}, \texttt{add\_riesz}, \texttt{relu} and \texttt{dpfp} are scale invariant, i.e., $\Phi_\tau=\tau^\alpha \Phi_1$ for all $\tau>0$ and some $\alpha>0$.
Therefore, the scale cancels out within attention and for $\varepsilon>0$ it merely rescales $\varepsilon$, so no tuning is necessary.
As the only not nonnegative kernel, \texttt{tri} causes numerical instability during training.
To stabilize this, we clamp the attention normalization at $10^{-6}$ and use gradient clipping to avoid overshooting.

\subsection{Associative recall capacity}
\label{subsec:associative_recall}

Here, we assess the capacity of various kernels by an associative-recall experiment similar to \citet[Sec.~6.1]{sis2021FastWeight}.
More precisely, we start with learnable tokens $\bm u=(u_n)_{n=1}^N\subseteq\R^D$ and define keys and queries as $k_n=W_K u_n$ and $q_n=W_Q u_n$ for learnable matrices $W_Q,W_K\in\R^{D\times D}$. Then, for values $v_1,\ldots,v_N\in \R^N$ chosen as the unit vectors, we minimize the loss function
\begin{equation}\label{eq:loss_function}
\mathcal L(\bm u, W_Q,W_K)=\mathbb{E}_{\sigma,\xi}\left[\sum_{n=1}^N \|v_{\xi(n)}- y_{\sigma(n)}\|^2\right],
\end{equation}
where $y_m =A(\bm q, \bm k_\sigma)\bm v_{\xi} = \sum_{l=1}^N A_{m,l}(\bm q,\bm k_\sigma) v_{\xi(l)}$ with $\bm k_\sigma = (W_K u_{\sigma(n)})_{n=1}^N$ is the attention vector as defined in \eqref{eq:kernel_attention} and $\sigma$ and $\xi$ are uniformly randomly drawn permutations.
The loss does not depend on $\sigma$ and $\xi$ as the following lemma shows.
\begin{lemma}\label{lem:recall_frobenius}
For all  permutations $\sigma$ and $\xi$ it holds
\begin{equation}
    \sum_{n=1}^N \|v_{\xi(n)} - y_{\sigma(n)}\|^2_2=\|A(\bm q,\bm k)-\mathrm{Id}_N\|_F^2,
\qquad \bm{k}\coloneqq(W_Ku_n)_{n=1}^N.
\end{equation}
\end{lemma}
While \citet{sis2021FastWeight} phrase this as a retrieval task, we show in Appendix~\ref{app:sec:associative_retrival} that it can be related to our capacity notion in the sense that 
\begin{equation}\label{eq:loss_bound}
    \inf_{\bm u, W_Q,W_K}\mathcal L(\bm u, W_Q,W_K)\geq \inf_{\bm q,\bm k} \|A(\bm q,\bm k)-\mathrm{Id}_N\|_F^2.
\end{equation}
In particular, we have that $\inf_{\bm u, W_Q,W_K}\mathcal L(\bm u, W_Q,W_K)=0$ only if $\mathrm{Cap}(\Phi)\geq N$.

\begin{figure}[t]
    \centering
    \includegraphics[width=\linewidth]{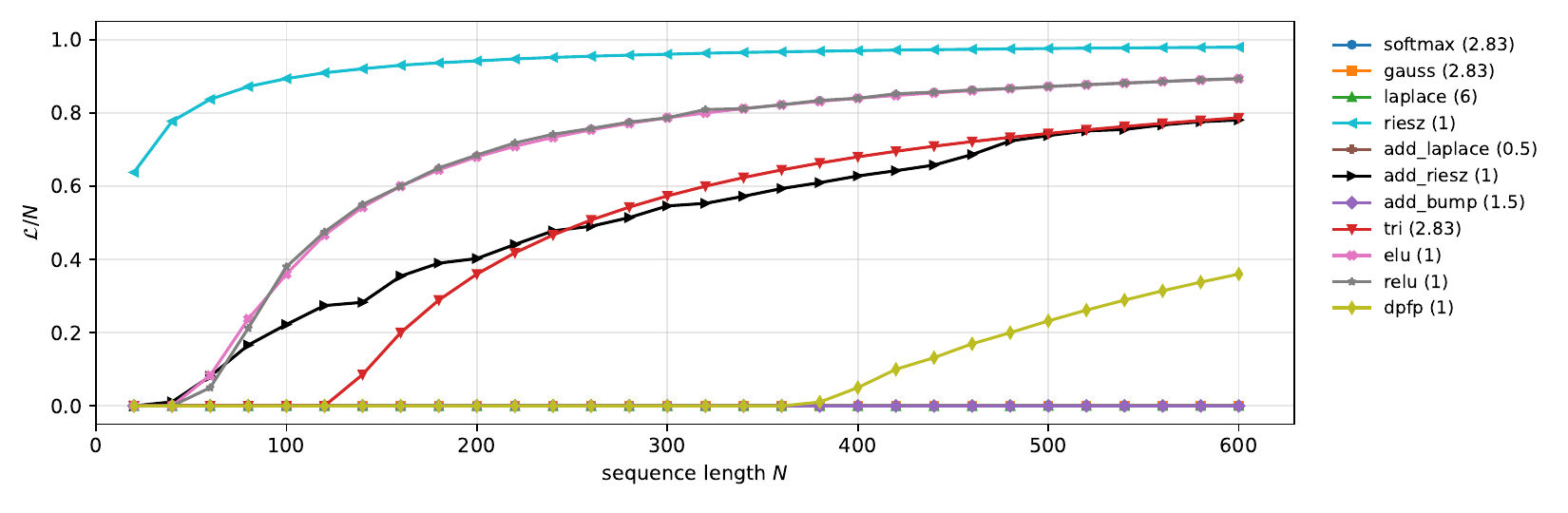}
    \caption{Associative-recall: Trained full-batch with Adam (lr $0.03$) for at most $120k$ steps, stopping early once the loss falls below $10^{-4}$, $D = 64$, where  the tokens are initialized as standard normal and $W_Q = W_K = I$. Curves show the best loss over $3$ seeds. Scale $\tau$ in brackets behind  kernel name.}
    \label{fig:associative_recall}
\end{figure}

In Figure~\ref{fig:associative_recall}, we plot the minimal value of $\mathcal L(\bm u, W_Q,W_K)$ after 120k optimization steps on $W_Q$, $W_K$ and $\bm u$ for the kernels of Table~\ref{tab:kernel_overview}. 
We observe that all FFM kernels roughly achieve a loss of zero for sequence lengths $N < L$ and end up with a strictly positive loss for $N > L$.
This matches Corollary~\ref{corr:ffm_finite_cap} stating that $\mathrm{Cap}(\Phi)\leq L$.
On the other hand, the stationary $\mathcal C_0$ kernels \texttt{gauss} and \texttt{laplace} succeed in this task as predicted by Theorem~\ref{thm:inf_cap}.
In alignment with Corollary~\ref{corr:add_bump_capacity}, \texttt{add\_laplace} and \texttt{add\_bump} also pass the task up to $N=600$.
While Theorem~\ref{thm:additive_capacity} guarantees $\mathrm{Cap}(\Phi)\geq128$ for \texttt{add\_riesz}, the optimization already fails around $N=40$, and \texttt{riesz} performs worst throughout.

\subsection{Implementations and Speed benchmarks}
\label{subsec:num_runtime}

Now, we compare the runtime of the different backends of the attention algorithm.
We already know that for $M=N$ softmax attention has complexity $\mathcal O(N^2)$, while our kernels have complexity $\mathcal O(N\log N)$.
Hence, for long enough sequence lengths, quasi linear kernels should eventually become faster.
Thus, we are interested in the ``break-even'' point, i.e., the minimal number $N$ such that the quasi linear algorithms are faster than the standard softmax attention. 
For the quasi linear algorithms, we provide backends for \texttt{add\_laplace} and the distance kernel. The latter can be used to compute all spline kernels and we choose \texttt{add\_riesz} as a simple representative.
Here, we consider the following backends.
\begin{itemize}[noitemsep, topsep=0pt]
    \item \texttt{GLOBAL}: Sorting $\bm q, \bm k$ and computing the prefix sums involves two separate CUDA kernels. Same structure is implemented for both \texttt{add\_riesz} and \texttt{add\_laplace}.
    \item \texttt{FUSED}: all operations are in a single kernel; this is limited by memory to $N\leq 4096$.
    \item \texttt{KEOPS}: a GPU-friendly $\mathcal O(N^2)$ reference \citep{CFGCD2021KeOps}.
\end{itemize}
We provide implementation details in Appendix~\ref{app:sec:implementation}.
All backends of the quasi linear attention algorithm run with \texttt{fp32} precision. While an \texttt{fp16} implementation is algorithmically possible and can lead to further speedups, it requires substantial code engineering, which goes beyond the scope of this paper.
We comment in Appendix~\ref{app:subsec:causal_masking} on this issue.
We compare the evaluation times of this kernel attention with the standard softmax attention.
Here, we consider the softmax backends:
\begin{itemize}[noitemsep, topsep=0pt]
    \item \texttt{MEM\_EFF32} and \texttt{MEM\_EFF16} \citep{xFormers2022} in \texttt{fp32} and \texttt{fp16} precision.
    \item \texttt{FLASH16} and \texttt{CUDNN16} \citep{dao2024flash2} using \texttt{fp16} precision.
\end{itemize}
Our \texttt{weighted\_abs\_sum} supports gradients and padding, such that all spline kernels can be used in both training and inference.
Quasi linear causal attention is possible and explained in Appendix~\ref{app:subsec:causal_masking} but not implemented. The \texttt{add\_laplace} is forward only.

\begin{table}
    \centering
    \caption{Forward-pass of kernel attention \eqref{eq:kernel_attention}  with runtime in milliseconds, mean over 10 runs after 5 warm-up iterations. Relative standard deviations are below $18\%$ for $N\leq 512$, below $4\%$ at $N=1024$ and below $3\%$ for $N\geq2048$. Bold marks the fastest method overall; underline marks the fastest among the five \texttt{fp32} methods. Shape: $B=4$, $H=12$, $D=64$, $C=64$, $M=N$.}
    \scriptsize
    \begin{tabular}{rcccccccc}
    \toprule
    & \texttt{add\_laplace} \texttt{fp32} & \multicolumn{3}{c}{\texttt{add\_riesz} \texttt{fp32}} & \multicolumn{4}{c}{\texttt{softmax}} \\
    \cmidrule(lr){2-2}\cmidrule(lr){3-5}\cmidrule(lr){6-9}
    $N$ & \texttt{GLOBAL} & \texttt{GLOBAL} & \texttt{FUSED} & \texttt{KEOPS} & \texttt{MEM\_EFF32} & \texttt{MEM\_EFF16} & \texttt{FLASH16} & \texttt{CUDNN16} \\
    \midrule
    $128$    & $0.442$ & $0.455$ & $0.305$ & $0.381$ & $\underline{0.029}$ & $\mathbf{0.021}$ & $\mathbf{0.021}$ & $0.023$ \\
    $256$    & $0.575$ & $0.539$ & $0.353$ & $0.425$ & $\underline{0.058}$ & $0.031$ & $\mathbf{0.022}$ & $0.030$ \\
    $512$    & $0.824$ & $0.694$ & $0.478$ & $0.688$ & $\underline{0.144}$ & $0.058$ & $\mathbf{0.041}$ & $0.043$ \\
    $1024$   & $1.305$ & $1.014$ & $0.719$ & $1.648$ & $\underline{0.431}$ & $0.152$ & $\mathbf{0.093}$ & $0.119$ \\
    $2048$   & $2.260$ & $1.690$ & $\underline{1.278}$ & $4.754$ & $1.645$ & $0.537$ & $0.266$ & $\mathbf{0.265}$ \\
    $4096$   & $4.472$ & $3.583$ & $\underline{2.991}$ & $17.14$ & $6.219$ & $1.948$ & $1.016$ & $\mathbf{1.001}$ \\
    $8192$   & $9.109$ & $\underline{7.336}$ & -- & $63.94$ & $24.12$ & $7.424$ & $3.780$ & $\mathbf{3.726}$ \\
    $16384$  & $20.31$ & $\underline{16.86}$ & -- & $249.1$ & $94.95$ & $29.57$ & $14.99$ & $\mathbf{14.90}$ \\
    $32768$  & $42.07$ & $\mathbf{\underline{34.88}}$ & -- & $985.6$ & $379.4$ & $118.8$ & $59.79$ & $60.11$ \\
    $65536$  & $86.94$ & $\mathbf{\underline{73.76}}$ & -- & $3877$ & $1523$ & $478.1$ & $239.2$ & $241.3$ \\
    $131072$ & $\mathbf{\underline{190.1}}$ & $200.0$ & -- & $15608$ & $6118$ & $1930$ & $968.8$ & $972.1$ \\
    \bottomrule
    \end{tabular}
    \label{tab:fwd_runtimes}
\end{table}

We report the runtime of a forward pass of \eqref{eq:kernel_attention} in Table~\ref{tab:fwd_runtimes}. Appendix~\ref{app:sec:implementation} covers the backward-pass with similar conclusions.
At equal precision \texttt{FUSED} becomes faster than \texttt{MEM\_EFF32} at $N\approx 2048$.
For $N=131072$, \texttt{GLOBAL} is roughly $31\times$ faster.
The \texttt{fp16} backends are $2$--$6\times$ faster than their \texttt{fp32} counterparts, so the break-even against \texttt{FLASH16} occurs slightly beyond $N=16384$.
We view this as an engineering rather than a conceptual gap: the \texttt{fp16} backends exploit tensor cores and years of kernel tuning, whereas a flash-style \texttt{fp16} tiling of the quadratic regime combined with our quasi linear kernel is possible in principle and would move the crossover accordingly.
Regarding memory, our backend allocates roughly $3\times$ the forward peak memory of \texttt{MEM\_EFF32}, as it must hold sorted copies of the queries, keys and their sorting permutations.
Unlike a naive PyTorch implementation of Algorithm~\ref{alg:weighted_abs_sum_forward}, which must materialize the prefix sums of size $\mathcal O(PDNC)$, we compute the prefix sums on the fly, such that the overhead reduces to $\mathcal O(PDN)$.

\subsection{Text embedding transformer}
\label{subsec:num_text}

Finally, we evaluate the performance of the quasi linear kernels on a text embedding transformer.
To this end, we consider the sentence embedder nomic-embed-text-v1 \citep{NMDM2024nomic} ($12$ layers, $12$ heads, $D=64$, rotary positional embeddings (RoPE)).

\paragraph{Setup and Training}
To avoid the computationally expensive and hyperparameter-sensitive training procedure, we start with a pretrained teacher model using softmax attention and distill student models using other kernels from this teacher model. Here, we proceed in three phases:
\begin{enumerate}[noitemsep, topsep=0pt]
    \item \textbf{Initialization:} We initialize the student model by copying the weights of the teacher model.
    \item \textbf{Layer-by-layer distillation:} For each attention layer, we match the MSE between the softmax teacher and the student.
    \item \textbf{End-to-end distillation:} We minimize the MSE between the end-to-end application of the teacher and student model. 
\end{enumerate}

For both distillations, we use \texttt{nomic-ai/nomic-embed-unsupervised-data}, specifically the \texttt{reddit\_title\_body} split.
Texts are truncated to $512$ tokens, batched in groups of $16$, and decorated with one of the task prefixes \texttt{classification}, \texttt{clustering}, \texttt{search\_document} and \texttt{search\_query}, sampled with probabilities $0.45$, $0.1$, $0.35$ and $0.1$.
In the first distillation, each layer is trained for $10{,}000$ steps with Adam, with a learning rate
cosine-annealed from $5\cdot10^{-4}$ to $5\cdot10^{-5}$.
In the second distillation, all parameters are trained for
$10{,}000$ steps with Adam, with a linear warmup over the first $10\%$ of steps to $5\cdot10^{-5}$, followed by
cosine annealing to $10^{-5}$.

\begin{figure}
\vspace{-.5cm}

\begin{minipage}{.48\linewidth}
\begin{table}[H]
    \centering
\caption{Results after full two-stage distillation for MTEB and LoCo.
Best performance among quasi linear kernels is underlined.}
    \resizebox{\textwidth}{!}{%
    \begin{tabular}{llcccc}
    \toprule
     &  & MTEB $\uparrow$ & \multicolumn{3}{c}{LoCo (nDCG@10) $\uparrow$}  \\
    \cmidrule(lr){3-3}
    \cmidrule(lr){4-6}
    Kernel & scale $\tau$ & $\leq 512$ & $2048$ & $4096$ & $8192$ \\
    \midrule
    \texttt{softmax} (teacher) & $2.828$ & $64.41$ & $0.873$ & $0.882$ & $0.883$ \\
    \texttt{gauss}       & $2.828$ & $64.41$ & $0.875$ & $0.884$ & $0.886$ \\
    \texttt{laplace}     & $6.0$     & $64.38$ & $0.885$ & $0.891$ & $0.887$ \\
    \texttt{riesz}       & $1.0$   & $61.69$ & $0.805$ & $0.839$ & $0.858$ \\\midrule
    \texttt{add\_riesz}  & $1.0$   & $61.72$ & $0.769$ & $0.813$ & $0.832$ \\
    \texttt{add\_laplace}& $0.5$   & $62.12$ & \underline{$0.830$} & \underline{$0.845$} & $0.858$ \\
    \texttt{add\_bump}   & $1.5$   & $62.12$ & $0.829$ & $0.843$ & \underline{$0.859$} \\
    \texttt{tri}         & $2.828$ & $61.53$ & $0.829$ & $0.838$ & $0.851$ \\
    \texttt{relu}        & $1.0$   & $62.29$ & $0.472$ & $0.516$ & $0.556$ \\
    \texttt{elu}         & $1.0$   & $61.55$ & $0.699$ & $0.773$ & $0.794$ \\
    \texttt{dpfp}        & $1.0$   & \underline{$63.42$} & $0.681$ & $0.811$ & $0.823$ \\
    \bottomrule
    \end{tabular}
    }
    \label{tab:text}
\end{table}
\end{minipage}\hfill
\begin{minipage}{.49\linewidth}
\begin{figure}[H]
    \centering
    \includegraphics[width=\linewidth]{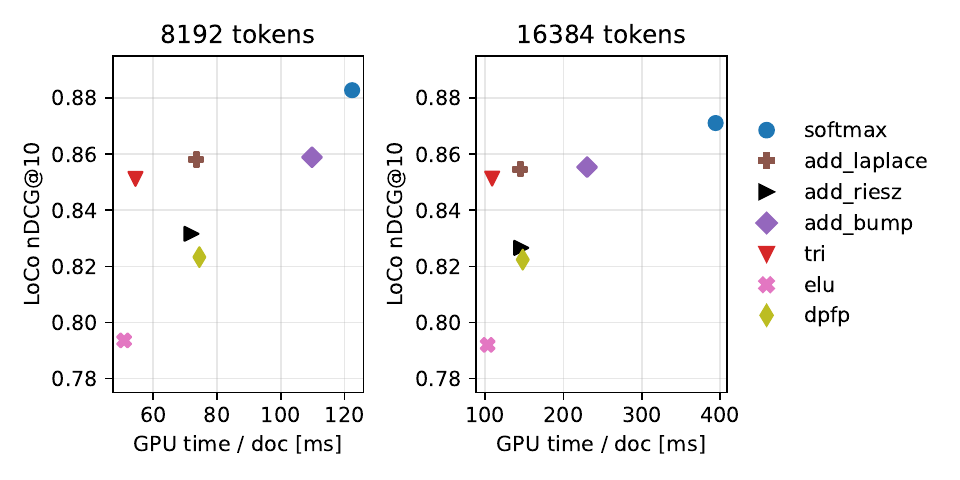}

    \vspace{-.5cm}
    \caption{LoCo retrieval quality against encoding cost at $8192$ and $16384$ tokens.
    Each point is one kernel: $y$ is the mean nDCG@10 over five LoCo tasks, $x$ is GPU time per document.
    }
    \label{fig:pareto}
\end{figure}
\end{minipage}
\end{figure}

\paragraph{Evaluation and Results} We evaluate on two regimes: the MTEB(eng, v2) benchmark \citep{MTMR2023mteb}, an average over $41$ datasets whose documents are short ($N\leq 512$), and LoCo long-context retrieval \citep{SFAGR2024loco} (nDCG@10, averaged over $5$ subsets) at lengths $2048$, $4096$ and $8192$, well beyond the training length of $512$.

We observe that \texttt{gauss} and \texttt{laplace} fully recover the teacher.
The Riesz kernel and the quasi linear kernels moderately lose performance.
The FFM kernels work well for the short sequences from the MTEB benchmark, but mostly degrade in performance for the LoCo.
Only the \texttt{tri} kernel achieves reasonable results.
On the other hand, the infinite capacity kernels \texttt{add\_bump} and \texttt{add\_laplace} remain close to the \texttt{softmax} teacher model. 
Moreover, on LoCo, stationary kernels tend to extrapolate better than non-stationary ones.
This may be due to RoPE (with NTK scaling), which is designed for rotation-invariant kernels such as \texttt{softmax}.
While \texttt{laplace}, \texttt{tri} and our additive kernels are not rotation invariant, they remain stationary, whereas the FFM kernels \texttt{relu}, \texttt{elu} and \texttt{dpfp} are neither.
For this retrieval task, which relies on few relevant tokens and mean pooling, stationarity appears to matter more for extrapolation than capacity.

Further, we evaluate the time-performance trade-off. More precisely, we plot in Figure~\ref{fig:pareto} the performance on LoCo versus the required evaluation time for each kernel.
We consider both sequence lengths of $8192$ and $16384$, where longer sequences are truncated and shorter sequences are padded.
We can see that for long sequences the \texttt{tri} and \texttt{add\_laplace} kernels both significantly improve the computation time compared to the softmax kernel with a moderate loss of performance. Here, the \texttt{tri} kernel is slightly faster while the \texttt{add\_laplace} kernel has a slightly better performance. 

\section{Conclusions, Limitations and Future Work}
We analyzed kernel attention to identify kernels that are both expressive and enable quasi linear computation.
Expressivity is measured through the maximum sequence length for which the attention matrix can approximate the identity (capacity).
We showed that expressive kernels such as softmax, Gauss and Laplace have infinite capacity while finite feature map kernels saturate.
By leveraging additive kernels and one dimensional sorting, we identified a class of kernels with infinite capacity and quasi linear complexity.
Experiments confirm that our CUDA-based kernel attention is faster than softmax attention for sequence lengths exceeding $2048$ elements in \texttt{fp32}.

While our capacity metric provides a rigorous theoretical framework for comparing kernel attention, it captures only one aspect of expressiveness.
Theoretically, kernel attention with higher capacity can represent classes of functions that cannot be represented with lower capacity.
However, the full expressive power of transformers is influenced by additional properties beyond capacity, such as stationarity and locality.
In the long term, an \texttt{fp16} implementation, leveraging tensor cores and optimized memory access patterns, could yield significant additional speedups, particularly for long sequences.
Further, we intend to extend our implementation to causal attention and to evaluate our method on larger-scale benchmarks to validate its scalability and performance.

\subsection*{AI use statement}
We did not use generative AI for any tasks with required disclosure: all research ideas, methodology, experimental design, and interpretation of results are the authors' own work, as is the design of the GPU implementation, including the parallelization strategy, the memory and thread layout and the kernel fusion; synthetic data generation, translation, and dataset cleaning are not applicable. Starting from a complete, author-written implementation of this design, we used generative AI tools to extend and clean up the code and repository, to unify naming and structure, and to apply hardware-specific low-level optimizations; these changes affect how the computation is executed on the device rather than the algorithm itself, they did improve performance, and the runtimes we report are those of the resulting code. We also used generative AI tools to identify related work and candidate citations. All AI-modified code was reviewed by the authors and verified for numerical agreement with the author-written reference implementation before being re-tested and re-benchmarked, and all AI-suggested references were verified against the original sources. The paper was written by the authors, with generative AI used to polish wording. We take responsibility for the final content of this work, including text, claims or artifacts produced with the aid of generative AI.

\subsection*{Reproducibility statement}
All required datasets and models are openly available from the literature.
The source code for our experiments is available at:
\begin{center}
    \url{https://github.com/Nicolaj-Rux/quasi-linear-kernel-attention}
\end{center}

\subsection*{Acknowledgments}
JH acknowledges funding from the German Research Foundation (DFG) within project no 530824055 and 572825596. NR acknowledges funding from the European Union and the Free State of Saxony through the European Social Fund (ESF). We are grateful to Michel Steuwer, Nicole Thalia Heinimann, and Dominic Arne Stöcker for their guidance on the CUDA implementation, and to Paul Hagemann for valuable advice on model training.

\bibliography{iclr2027_conference}
\bibliographystyle{iclr2027_conference}

\appendix

\section{Background on FFMs}
\label{app:sec:FFM}
Here, we discuss the feature maps proposed by \cite{KVPF_Trans_are_RNNs, performer2021, PPYSSK2021RFA, sis2021FastWeight}.
In \cite{KVPF_Trans_are_RNNs}, the feature map $\varphi\colon \R^D\to \R^D$ with $\varphi(x)\coloneqq [\mathrm{elu}(x_1)+1,\ldots, \mathrm{elu}(x_D)+1]$ is used.
Later, \cite{PPYSSK2021RFA} used Bochner's theorem to construct a feature map that approximates the Gauss kernel via
\begin{equation}
\varphi_{\cos}\colon \R^D\to \R^{2L}, \varphi_{\cos}(x)\coloneqq \nicefrac{1}{\sqrt{L}}\bigl[\sin(\omega_1^\top x),\ldots, \sin(\omega_{L}^\top x),\cos(\omega_1^\top x),\ldots, \cos(\omega_{L}^\top x)\bigr],
\end{equation}
where $\omega_1,\ldots, \omega_{L}\sim \mathcal N(0,\tau I_D)$ are sampled i.i.d.
The associated kernel is then given as
\begin{align}
\varphi_{\cos}(q)^\top \varphi_{\cos}(k) &= \frac{1}{L}\sum_{l=1}^{L} \bigl(\sin(\omega_l^\top q)\sin(\omega_l^\top k)+\cos(\omega_l^\top q)\cos(\omega_l^\top k)\bigr) \notag \\
& =\frac{1}{L}\sum_{l=1}^{L} \cos\bigl(\omega_l^\top (q-k)\bigr).
\end{align}
Now, Bochner's theorem implies for $z=q-k$ the relation
\begin{align}\label{eq:bochner_rfa}
    G(q,k)&
    =(2\pi \tau)^{-\frac{D}{2}}\int_{\R^D} \cos(\omega^\top z) \e^{-\frac{\|\omega\|^2}{2\tau}} \d \omega
    =\mathop{\mathbb{E}}\limits_{\omega \sim \mathcal{N}(0,\tau I_D)} \bigl[\cos(\omega^\top z)\bigr]
     \approx \frac{1}{L} \sum_{l=1}^{L} \cos( \omega_l^\top z).
\end{align}
This idea is due to \cite{RR2007RandomFeatures}, who also establish error estimates.
The kernel induced by this feature map is not nonnegative anymore, which can lead to numerical instability.

In \cite{PPYSSK2021RFA}, they resample $\omega_l$ during training, but keep them fixed at test time, as they did not notice any change of behavior when resampling.
Moreover, they use exactly $L=D$ directions.
Attention generates the queries $q$, keys $k$ and values $v$ by multiplying a token $u\in \R^D$ with three separate matrices $W_Q, W_K, W_V\in \R^{D\times D}$.
The linear projections $\omega_1,\ldots, \omega_D$ can be fused with the matrices $W_Q$ and $W_K$ to $\omega W_K$ and $\omega W_Q$, and we can assume instead the form \smash{$\nicefrac{1}{L}\sum_{l=1}^{L} \cos(q_l-k_l)$ with $L=D$}.
Therefore, we interpret the resampling during training as a training mechanism rather than a property of the kernel or attention.

\citet{performer2021} use a similar argument to directly approximate the softmax kernel with the feature map $\varphi_{\exp}(x)=\exp(-\nicefrac{\|x\|_2^2}{2})[\exp(x_1),\ldots, \exp(x_D)]$. With the same argument for $\varphi_{\cos}$ from the previous paragraph we omit the random projections.
In contrast to $\varphi_{\cos}$, their FFM $\varphi_{\exp}$ is nonnegative, but not stationary, see Proposition~\ref{prop:stationary_FFM}.
Among the FFMs introduced by \cite{performer2021}, they found  $\varphi_\mathrm{ReLU}(x)=[\mathrm{ReLU}(x_1),\ldots, \mathrm{ReLU}(x_D)]$ to generally perform best, which is why we use $\varphi_\mathrm{ReLU}$ during the numerical comparison.

Table~\ref{tab:feature_maps} summarizes the previous FFMs $\varphi\colon \R^D\to \R^{J D}$ of \cite{KVPF_Trans_are_RNNs, performer2021, PPYSSK2021RFA}, which all are of the form
\begin{equation}\label{eq:feature_map_form}
    \varphi(x)=[f_1(x_1), \ldots, f_1(x_D), \ldots, f_J(x_1), \ldots, f_J(x_D)]^\top\in \R^{J D}.
\end{equation}
\begin{table}[H]
    \centering
    \caption{Examples of FFMs of the form \eqref{eq:feature_map_form}.}
    \begin{tabular}{llll}
\toprule
         Paper & $\varphi$ & $J$ & $f_j$  \\
\midrule
         \cite{KVPF_Trans_are_RNNs}& $\varphi_\mathrm{elu}$  & $1$ & $f_1(x)= \mathrm{elu}(x)+1$\\
         \cite{PPYSSK2021RFA}& $\varphi_\mathrm{tri}$ &  $2$ & $f_1(x)= \sin(x), f_2(x)=\cos(x)$\\
         \cite{performer2021}&$\varphi_\mathrm{ReLU}$ & $1$ & $f_1(x)= \mathrm{ReLU}(x)$\\
\bottomrule
    \end{tabular}
    \label{tab:feature_maps}
\end{table}
Note that all FFM of the form \eqref{eq:feature_map_form} are additive kernels over $F(s,t)=f(s)^\top f(t)$, because
\begin{equation}
\Phi(q,k)= \varphi(q)^\top \varphi(k)
= \sum_{d=1}^D F(q_d, k_d),\quad\text{ with } \quad  F(s,t)=f(s)^\top f(t).
\end{equation}
However, their univariate kernel $F$ is a FFM with $\dim \mathcal H_F=J$ of dimension  $1$ or $2$, while the Laplace, Riesz and bump kernel have infinite dimension.

\citet{sis2021FastWeight} replaces the random features proposed by \cite{performer2021,PPYSSK2021RFA} with \emph{deterministic parameter-free projection} (DPFP): Let $r\colon \R^D\to \R^{2D}$ with $r(x)=[\mathrm{ReLU}(x),\mathrm{ReLU}(-x)]\in \R^{2D}$ and $\nu\in \{1,\ldots, D{-}1\}$.
The DPFP feature map is $\varphi\colon \R^D\to \R^{2D\nu}$ with $\varphi_{2Dj+k}(x) = r(x)_k\, r(x)_{k+j}$ for $j=0,\ldots, \nu-1$ and $k=1,\ldots, 2D$, where the indices are taken modulo $2D$.
This differs from the general form \eqref{eq:feature_map_form}.
Here, we require $\nu<D$, so that the components of $\varphi$ are linearly independent.
In fact, $\nu\geq D$ is wasteful, as it does not increase the dimension of the RKHS anymore.
In practice, we follow \citet{sis2021FastWeight} and use $\nu= 3$ with $D=64$, so that $\dim \mathcal H = 2\nu D=384$.

\section{Connection to slicing}
\label{app:sec:slicing}
Additive kernels can be seen as a deterministic slicing of radial kernels.
Given a univariate kernel $f(|s-t|)$, the sliced kernel identity by \cite{Johannes2024KernelSum} expresses a radial kernel on $\R^D$ as an expectation over uniformly random directions,
\begin{equation}\label{eq:slicing}
     F(\|q-k\|_2)=\mathbb E_{\xi\sim\mathcal U(\mathbb S^{D-1})}\bigl[f(|\langle q-k,\xi\rangle|)\bigr]\approx \frac{1}{L}\sum_{l=1}^L f(|\langle q-k,\xi_l\rangle|),
\end{equation}
where $F$ is obtained from $f$ by the Riemann Liouville transform $\mathcal S_d$, see also \citet[Eq.~(10)]{rhs2025slicing}.
The fast summation of \cite{Johannes2024KernelSum} estimates \eqref{eq:slicing} by drawing $L$ directions $\xi_1,\ldots,\xi_L$ and applying a one dimensional fast summation along each of them, which costs $\mathcal O(L(N{+}M)\log N)$ and converges at the Monte Carlo rate \smash{$\mathcal O(L^{-\nicefrac{1}{2}})$}.
Replacing the random directions by the fixed coordinate axes $e_1,\ldots,e_D$ yields, up to the factor $\nicefrac{1}{D}$, exactly the additive kernel,
\begin{equation}
  \frac{1}{L}\sum_{l=1}^L f(|\langle q-k,\xi_l\rangle|)= \frac{1}{D}\sum_{l=1}^D f(|\langle q-k,e_l\rangle|) = \frac{1}{D}\sum_{d=1}^D f(|q_d-k_d|)
\end{equation}
The two constructions therefore share the same one dimensional primitive.
Choosing for $L=D$ the directions $\xi_l=e_l$ can further be interpreted as quasi-Monte Carlo design for the integration of the sphere, which sometimes achieves better error rates than the previously mentioned standard estimate for Monte Carlo methods. We refer to \cite{QMC_slicing} for details.

In the special case of Riesz, the sliced multivariate Riesz kernel yields the univariate Riesz kernel.
Therefore, additive Riesz $\|q\|_1+\|k\|_1-\|q-k\|_1$ can be seen as a Monte Carlo approximation of the multivariate Riesz kernel  $\|q\|_2+\|k\|_2-\|q-k\|_2$.
However, the additive Laplace kernel does not appear as Monte Carlo approximation neither of $\ell_1$ Laplace nor of $\ell_2$ Laplace.
Instead additive Laplace is the Monte Carlo approximation of a power series given in \citet[Table.~1]{Johannes2024KernelSum}.
For more background on slicing see \cite{Johannes2024KernelSum, rhs2025slicing}. Applications of the slicing algorithm were also presented in \cite{generative_sliced,boufadene2025fast}.

\section{Fast Computation of Spline-Kernel Attention}
\label{app:sec:implementation}
In most transformer architectures, the head dimension $D$ is either $32,64$ or $128$, and the number of channels $C$ usually equals $D$, or $C=D{+}1$ if the values are extended by a row of constant ones for normalization.
We are interested in the case where the sequence lengths satisfy $N\approx M$ and $N,M\gg D,C$.
Let $\bm q\in \R^{M\times D}$, $\bm k\in \R^{N\times D}$ and $\bm v\in \R^{N\times C}$.  We focus on the efficient computation of \emph{absolute value sums}
\begin{align}\label{eq:sliced_riesz_sum}
z_m =\sum_{n=1}^{N}\|q_m-k_n\|_1v_n, \qquad m=1,\ldots,M.
\end{align}
The naive computation of the absolute value sum requires $\mathcal O((D{+}C)NM)$ operations and would be the computational bottleneck if $N$ and $M$ are both large.
Instead, it is possible to reduce the complexity to $\mathcal O(D(N{+}M)(\log_2 N{+}C))$.
Simply, rewrite
\begin{equation}
    \sum_{n=1}^{N}\|q_m-k_n\|_1v_n=\sum_{d=1}^{D}\sum_{n=1}^{N} |q_{m,d}-k_{n,d}|v_{n}.
\end{equation}
Each inner sum is of the type $\sum_{n=1}^{N} |s_m-t_n|v_n$, where $s_m=q_{m,d}\in \R$ and $t_n=k_{n,d}\in \R$.
The absolute value sum in $\R$ can be computed in $\mathcal O((N{+}M)(\log_2 N{+}C))$ by Algorithm~\ref{alg:weighted_abs_sum_forward}.
By Proposition~\ref{prop:spline_is_quasi} all spline kernels with finitely many nodes, can be computed in quasi linear time, by calling Algorithm~\ref{alg:weighted_abs_sum_forward} on each node. More details are given inside the proof of Proposition~\ref{prop:spline_is_quasi}

\subsection{Parallelization and Memory overhead}
\label{app:subsec:parallelization}
In practice, an efficient scalable implementation is often worth more than theoretical runtime asymptotics.
Although flash attention \citep{dao2022flash} did not change the asymptotic work, it had significant impact.
While \eqref{eq:sliced_riesz_sum} can be computed asymptotically fast, two questions must be answered simultaneously for both the forward and the backward pass: (i) Can the implementation be \emph{parallelized}? and (ii) Can we handle \emph{memory efficiently}?
Without a parallel implementation, this method cannot compete against massively parallel algorithms such as flash attention.
Even a parallel implementation will not convince in practice if it is not memory efficient.

In practice, each tensor has an additional batch dimension $P$, namely $\bm q\in \R^{P\times D\times M}$, $\bm k\in \R^{P\times D\times N}$ and $\bm v\in \R^{P\times N\times C}$.
If we implement Algorithm~\ref{alg:weighted_abs_sum_forward} via PyTorch primitives, then the prefix sums are materialized for each $D$ and each batch $P$, resulting in $\mathcal O((N{+}M)DCP)$ memory costs.
This is a known issue that has been attacked already by \cite{dao2022flash, CFGCD2021KeOps} for certain expand-reduce computations.
Of course, we could instead serialize over $D$ and compute each part independently or in a smaller batch, but then again we pay with parallelism for memory.
Yet, there is a way to parallelize completely over $P$, $D$ and $C$, staying serial only in $N/M$, without materializing any huge intermediates.
The idea is to sort $\bm q$ as well and never materialize the entire prefix sums. Instead, for each $d$, we scan through the sorted keys $k_{\sigma(n),d}$ and values $v_{\sigma(n)}$ for $n=1,\ldots,N$, compute the prefix sums on the fly and write the result for query $m$ as soon as the scan reaches position $p_m$.
Since we also sort $\bm q$, the positions $p_m$ are monotonically increasing, which ensures that no previous values are needed again.
In the end, we correct the result with the total sum.
In practice, we use a grid of size $(D, P)$ and block size $C$.
Since a warp always works in $32$ threads, our implementation favors $C$ being a multiple of $32$ and $C\leq 1024$, which is usually satisfied.
The fused scan kernel has no thread divergence and data independent runtime.
One disadvantage is that the final reduction over $D$ cannot be performed within a block, so atomic stores are needed.
This is the main bottleneck of our implementation.

\subsection{Gradients and Autodiff}
\label{app:subsec:gradients}
Since our implementation is in CUDA, we need to implement the gradient for autograd ourselves. All gradients of our method can be computed in $\mathcal O((N{+}M)\log(N{+}M))$, such that both forward and backward run subquadratic in $N{+}M$.
The subgradient of the absolute value is $\sign$, and within the gradients it is necessary to compute sums of the form
\begin{align}
\sum_{n=1}^{N} \sign(s_m{-}t_n)v_n
=& \sum_{n=1}^{N} \sign(s_m{-}t_{\sigma(n)})v_{\sigma(n)}
=  \sum_{n=1}^{p_m} v_{\sigma(n)} - \sum_{n=p_m+1}^{N}  v_{\sigma(n)} \notag \\
\label{eq:weighted_sign_sum}
=& 2 \text{pref\_v}_{p_m}-\text{pref\_v}_{N} .
\end{align}
Again, these sums can be computed with $\mathcal O((N{+}M)(\log_2 N{+}C))$ work.

Let $\bm z\in \R^{M\times C}$ be the output of \eqref{eq:sliced_riesz_sum} and $\bm g\in \R^{M\times C}$ the upstream gradient.
In the following, we compute the vector Jacobian products (VJP) of $\mathcal{J} = \langle \bm g,\bm  z \rangle$. 
The VJP w.r.t. $\bm v\in \R^{N\times C}$ is
\begin{equation}
\frac{\partial \mathcal{J}}{\partial v_{n,c}} = \sum_{d=1}^{D}\sum_{m=1}^{M} |q_{m,d} - k_{n,d}|\, g_{m,c}.
\end{equation}
This gradient can be computed just like the forward pass by swapping $\bm q \leftrightarrow \bm k$ and replacing $\bm v$ by $\bm g$.
The VJP w.r.t. $\bm q\in \R^{M\times D}$ reads
\begin{equation}
\frac{\partial \mathcal{J}}{\partial q_{m,d}} = \sum_{c=1}^{C} g_{m,c} \sum_{n=1}^{N} \sign(q_{m,d} - k_{n,d})v_{n,c}.
\end{equation}
In theory, the tensor $\textrm{pref\_v}$ could be passed from the forward to the backward pass.
However, as we compute it on the fly, we need to recompute it during the backward pass.
Lastly, the VJP w.r.t. $\bm k\in \R^{N\times D}$ is given via
\begin{equation}
\frac{\partial \mathcal{J}}{\partial k_{n,d}} = \sum_{c=1}^{C} v_{n,c}\sum_{m=1}^{M} \sign(k_{n,d} - q_{m,d})\, g_{m,c}.
\end{equation}
Both sign sums have a reduction over $C$, which is our block dimension.
Therefore, this reduction can be performed within a block and no atomics are needed anymore.

\begin{table}[t]
    \centering
    \caption{Forward+backward  of kernel attention \eqref{eq:kernel_attention} with runtime in milliseconds, mean over 10 runs after 5 warm-up iterations, with forward and backward timed together in a single region. Relative standard deviations are below $10\%$ for $N\leq512$, below $1.2\%$ at $N=1024$ and below $0.7\%$ for $N\geq2048$. Bold marks the fastest method overall; underline marks the fastest among the four \texttt{fp32} methods. Shape: $B=4$, $H=12$, $D=64$, $C=64$, $M=N$.}
    \scriptsize
    \begin{tabular}{rccccccc}
    \toprule
    & \multicolumn{3}{c}{\texttt{add\_riesz} \texttt{fp32}} & \multicolumn{4}{c}{\texttt{softmax}} \\
    \cmidrule(lr){2-4}\cmidrule(lr){5-8}
    $N$ & \texttt{GLOBAL} & \texttt{FUSED} & \texttt{KEOPS} & \texttt{MEM\_EFF32} & \texttt{MEM\_EFF16} & \texttt{FLASH16} & \texttt{CUDNN16} \\
    \midrule
    $128$    & $1.066$ & $0.913$ & $1.598$ & $\underline{0.098}$ & $0.071$ & $0.059$ & $\mathbf{0.057}$ \\
    $256$    & $1.329$ & $1.147$ & $2.424$ & $\underline{0.233}$ & $0.104$ & $\mathbf{0.067}$ & $0.069$ \\
    $512$    & $1.832$ & $1.620$ & $4.087$ & $\underline{0.609}$ & $0.231$ & $\mathbf{0.132}$ & $0.143$ \\
    $1024$   & $2.950$ & $2.664$ & $11.96$ & $\underline{1.923}$ & $0.634$ & $\mathbf{0.324}$ & $0.363$ \\
    $2048$   & $5.283$ & $\underline{4.927}$ & $40.33$ & $7.360$ & $2.262$ & $\mathbf{0.978}$ & $1.050$ \\
    $4096$   & $11.19$ & $\underline{10.66}$ & $140.5$ & $27.95$ & $8.363$ & $\mathbf{3.731}$ & $4.029$ \\
    $8192$   & $\underline{23.18}$ & -- & $520.9$ & $109.0$ & $31.72$ & $\mathbf{13.90}$ & $14.99$ \\
    $16384$  & $\mathbf{\underline{50.03}}$ & -- & $2007$ & $433.4$ & $129.8$ & $54.35$ & $57.88$ \\
    $32768$  & $\mathbf{\underline{103.8}}$ & -- & $7903$ & $1729$ & $503.6$ & $214.9$ & $228.6$ \\
    $65536$  & $\mathbf{\underline{219.1}}$ & -- & $31550$ & $6905$ & $1990$ & $856.9$ & $913.2$ \\
    $131072$ & $\mathbf{\underline{600.6}}$ & -- & $125800$ & $27730$ & $8008$ & $3442$ & $3666$ \\
    \bottomrule
    \end{tabular}
    \label{tab:fwdbwd_runtimes}
\end{table}

Table~\ref{tab:fwdbwd_runtimes} measures the speed of our additive Riesz attention versus softmax backend.
For \texttt{FUSED}, only the forward pass is fused: the backward walk reuses the sorted tensors, which the fused forward does not materialize, so it runs on the \texttt{GLOBAL} backward.
The break even points are slightly better than the forward computation in Table~\ref{tab:fwd_runtimes}, as \texttt{GLOBAL} even beats \texttt{FLASH16} before $N=16384$. 
This has a simple reason: We measure the elapsed time for the forward and the $\bm q, \bm k$ and $\bm v$- backwards. Usually, those are $4$ equally sized computations. 
For example, \texttt{FLASH16}'s forward+backward at $N=131072$ is $3442$ and about $3.55\times$ larger than its forward $968.8$ in Table~\ref{tab:fwd_runtimes}. 
However, \texttt{GLOBAL}'s presort of $\bm q$ and $\bm k$, can be reused during the three backwards leading to a factor of $3$ from $200$ to $600.6$. 
Similarly, the relative memory overhead between \texttt{GLOBAL} and  \texttt{fp32} softmax backend is around $1.5$, while the forward was around $3$ times larger.

\subsection{Padding}
\label{app:subsec:padding}
Consider a batch of $P$ samples with sequence lengths $N_p$, $p=1,\ldots,P$, and maximal length $N=\max_p N_p$.
Padding extends the queries $\bm q\in\R^{D\times N_p}$, keys $\bm k\in\R^{D\times N_p}$ and values $\bm v\in\R^{C\times N_p}$ of sample $p$ to \smash{$\bm{\tilde q},\bm{\tilde k}\in\R^{D\times N}$} and \smash{$\bm{\tilde v}\in\R^{C\times N}$} with
\begin{equation}
    \tilde q_n=q_n,\, \tilde k_n=k_n,\, \tilde v_n =v_n\quad \text{for } n\leq N_p,
    \quad 
    \tilde q_n=0,\, \tilde k_n=0,\, \tilde v_n =0 \quad \text{for } N_p<n\leq N.
\end{equation}
Since the padded values vanish, the kernel sums \eqref{eq:kernel_sum} are unchanged,
$\tilde z_m=\sum_{n=1}^{N}\Phi(\tilde q_m,\tilde k_n)\tilde v_n=z_m$ for $m\leq N_p$,
so the weighted absolute value sums of Algorithm~\ref{alg:weighted_abs_sum_forward} are applied to the padded tensors without any modification.
Only the normalization \smash{$\sum_{n=1}^{N}\Phi(\tilde q_m,\tilde k_n)$} counts every padded key with weight one and hence contains the additional term $(N-N_p)\,\Phi(\tilde q_m,0)$.
In order to fix the normalization, we compute this correction for all queries in $\mathcal O(ND)$ operations, which is negligible compared to the kernel sums.
Then, we obtain for $m=1,\ldots,N_p$ that
\begin{align}
    \frac{\sum_{n=1}^{N_p}\Phi(q_m,k_n)v_n}{\sum_{n=1}^{N_p}\Phi(q_m,k_n)}
    =\frac{\sum_{n=1}^{N}\Phi(\tilde q_m,\tilde k_n)\tilde v_n}{\sum_{n=1}^{N}\Phi(\tilde q_m,\tilde k_n)-(N-N_p)\,\Phi(\tilde q_m,0)},
\end{align}
while the outputs of the padded queries $m>N_p$ are discarded.
Thus, we can pad samples of different length to a common length $N$ in order to process inputs of different sizes in one batch, which is necessary, for example, for text models.

\subsection{Causal masking}
\label{app:subsec:causal_masking}

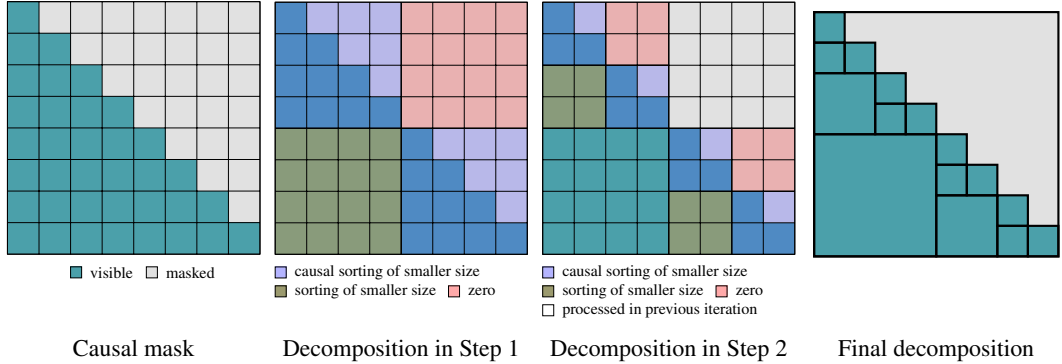
\begin{figure}[H]
\centering
\definecolor{allowed}{RGB}{75,160,170}
\definecolor{masked}{RGB}{225,225,225}
\begin{subfigure}{.24\textwidth}
\resizebox{\textwidth}{!}{%
\begin{tikzpicture}[
    every node/.style={font=\small},
    allowedcell/.style={fill=allowed},
    maskedcell/.style={fill=masked}
]
    \fill[maskedcell] (0,0) rectangle (8,8);
    \foreach \r in {0,...,7} {
        \pgfmathtruncatemacro{\y}{7-\r}
        \foreach \c in {0,...,\r} {
            \fill[allowedcell] (\c,\y) rectangle ++(1,1);
        }
    }

    \draw[thin] (0,0) grid (8,8);
    \draw[thick] (0,0) rectangle (8,8);

    \fill[allowedcell] (2,-.75) rectangle ++(.35,.35);
    \draw (2,-.75) rectangle ++(.35,.35);
    \node[anchor=west] at (2.5,-.575) {\Large visible};

    \fill[maskedcell] (4.4,-.75) rectangle ++(.35,.35);
    \draw (4.4,-.75) rectangle ++(.35,.35);
    \node[anchor=west] at (4.9,-.575) {\Large masked};

    \node[anchor=west] at (.5,-1.775) {\Large \phantom{processed in previous iteration}};
\end{tikzpicture}}
\caption*{Causal mask}
\end{subfigure}\hfill
\begin{subfigure}{.24\textwidth}
\resizebox{\textwidth}{!}{%
\begin{tikzpicture}[
    every node/.style={font=\small},
    allowedcell/.style={fill=allowed},
    maskedcell/.style={fill=masked},
    diagblock/.style={fill=blue!65, fill opacity=.30},
    upperblock/.style={fill=red!75, fill opacity=.30},
    lowerblock/.style={fill=orange!85, fill opacity=.35}
]
    \fill[maskedcell] (0,0) rectangle (8,8);
    \foreach \r in {0,...,7} {
        \pgfmathtruncatemacro{\y}{7-\r}
        \foreach \c in {0,...,\r} {
            \fill[allowedcell] (\c,\y) rectangle ++(1,1);
        }
    }

    \fill[diagblock]  (0,4) rectangle (4,8);
    \fill[diagblock]  (4,0) rectangle (8,4);
    \fill[upperblock] (4,4) rectangle (8,8);
    \fill[lowerblock] (0,0) rectangle (4,4);

    \draw[thin] (0,0) grid (8,8);
    \draw[thick] (0,0) rectangle (8,8);

    \draw[very thick] (4,0) -- (4,8);
    \draw[very thick] (0,4) -- (8,4);

    \fill[blue!65, fill opacity=.45] (0,-.75) rectangle ++(.35,.35);
    \draw (0,-.75) rectangle ++(.35,.35);
    \node[anchor=west] at (.5,-.575) {\Large causal sorting of smaller size};

    \fill[allowed] (0,-1.35) rectangle ++(.35,.35);
    \fill[orange!85, fill opacity=.45] (0,-1.35) rectangle ++(.35,.35);
    \draw (0,-1.35) rectangle ++(.35,.35);
    \node[anchor=west] at (.5,-1.175) {\Large sorting of smaller size};

    \fill[red!75, fill opacity=.45] (5.5,-1.35) rectangle ++(.35,.35);
    \draw (5.5,-1.35) rectangle ++(.35,.35);
    \node[anchor=west] at (6,-1.175) {\Large zero};

    \node[anchor=west] at (.5,-1.775) {\Large \phantom{processed in previous iteration}};
\end{tikzpicture}}
\caption*{Decomposition in Step 1}
\end{subfigure}\hfill
\begin{subfigure}{.24\textwidth}
\resizebox{\textwidth}{!}{%
\begin{tikzpicture}[
    every node/.style={font=\small},
    allowedcell/.style={fill=allowed},
    maskedcell/.style={fill=masked},
    diagblock/.style={fill=blue!65, fill opacity=.30},
    upperblock/.style={fill=red!75, fill opacity=.30},
    lowerblock/.style={fill=orange!85, fill opacity=.35},
    previous/.style={fill=gray!85, fill opacity=.}
]
    \fill[maskedcell] (0,0) rectangle (8,8);
    \foreach \r in {0,...,7} {
        \pgfmathtruncatemacro{\y}{7-\r}
        \foreach \c in {0,...,\r} {
            \fill[allowedcell] (\c,\y) rectangle ++(1,1);
        }
    }

    \fill[diagblock]  (0,6) rectangle (2,8);
    \fill[diagblock]  (2,4) rectangle (4,6);
    \fill[diagblock]  (4,2) rectangle (6,4);
    \fill[diagblock]  (6,0) rectangle (8,2);
    \fill[previous]   (4,4) rectangle (8,8);
    \fill[previous]   (0,0) rectangle (4,4);
    \fill[upperblock] (2,6) rectangle (4,8);
    \fill[upperblock] (6,2) rectangle (8,4);
    \fill[lowerblock] (0,4) rectangle (2,6);
    \fill[lowerblock] (4,0) rectangle (6,2);

    \draw[thin] (0,0) grid (8,8);
    \draw[thick] (0,0) rectangle (8,8);

    \draw[very thick] (4,0) -- (4,8);
    \draw[very thick] (0,4) -- (8,4);
    \draw[very thick] (2,4) -- (2,8);
    \draw[very thick] (6,0) -- (6,4);
    \draw[very thick] (0,6) -- (4,6);
    \draw[very thick] (4,2) -- (8,2);

    \fill[blue!65, fill opacity=.45] (0,-.75) rectangle ++(.35,.35);
    \draw (0,-.75) rectangle ++(.35,.35);
    \node[anchor=west] at (.5,-.575) {\Large causal sorting of smaller size};

    \fill[allowed] (0,-1.35) rectangle ++(.35,.35);
    \fill[orange!85, fill opacity=.45] (0,-1.35) rectangle ++(.35,.35);
    \draw (0,-1.35) rectangle ++(.35,.35);
    \node[anchor=west] at (.5,-1.175) {\Large sorting of smaller size};

    \fill[red!75, fill opacity=.45] (5.5,-1.35) rectangle ++(.35,.35);
    \draw (5.5,-1.35) rectangle ++(.35,.35);
    \node[anchor=west] at (6,-1.175) {\Large zero};

    \fill[previous] (0,-1.95) rectangle ++(.35,.35);
    \draw (0,-1.95) rectangle ++(.35,.35);
    \node[anchor=west] at (.5,-1.775) {\Large processed in previous iteration};
\end{tikzpicture}}
\caption*{Decomposition in Step 2}
\end{subfigure}\hfill
\begin{subfigure}{.24\textwidth}
\resizebox{\textwidth}{!}{%
\begin{tikzpicture}[
    every node/.style={font=\small},
    allowedcell/.style={fill=allowed},
    maskedcell/.style={fill=masked}
]
    \fill[maskedcell] (0,0) rectangle (8,8);
    \foreach \r in {0,...,7} {
        \pgfmathtruncatemacro{\y}{7-\r}
        \foreach \c in {0,...,\r} {
            \fill[allowedcell] (\c,\y) rectangle ++(1,1);
        }
    }

    \draw[line width=2pt] (0,0) rectangle (4,4);
    \draw[line width=2pt] (0,4) rectangle (2,6);
    \draw[line width=2pt] (4,0) rectangle (6,2);
    \draw[line width=2pt] (0,6) rectangle (1,7);
    \draw[line width=2pt] (2,4) rectangle (3,5);
    \draw[line width=2pt] (4,2) rectangle (5,3);
    \draw[line width=2pt] (6,0) rectangle (7,1);
    \foreach \i in {0,...,7} {
        \draw[line width=2pt] (\i,{7-\i}) rectangle ++(1,1);
    }
    \draw[line width=2pt] (0,0) rectangle (8,8);
    \path (-.15,-.15) rectangle (8.15,8.15);

    \node[anchor=west] at (.5,-1.775) {\Large \phantom{processed in previous iteration}};
\end{tikzpicture}}
\caption*{Final decomposition}
\end{subfigure}
\caption{Illustration of the causal mask and of Algorithm~\ref{alg:causal_sorting} for computing the kernel summation with causal mask in quasi linear time. From left to right: causal mask applied to the kernel matrix; block decomposition in the first step of Algorithm~\ref{alg:causal_sorting}; block decomposition in the second step; the resulting block partition after the full recursion.}
\label{fig:causal_mask}
\end{figure}

\begin{algorithm}[t]
\caption{Sorting algorithm for causal masked kernel sums}
\label{alg:causal_sorting}

\setcounter{algline}{0}
\begin{tabularx}{\linewidth}{@{}r X l l@{}}
\algrow \textbf{Input}
$\bm s\in \R^M$, $\bm t\in \R^N$ and $\bm v\in \R^{N\times C}$
\\
\algrow \textbf{Output} $(z_m)_{m=1}^M=\mathrm{CausalSorting}(\bm s, \bm t, \bm v)$ given by
$z_m \coloneqq \sum_{n=1}^{m} \phi(s_m,t_n)v_n\in \R^C$
\\
\algrow \textbf{Denote} by $\mathrm{Sorting}(\bm s,\bm t,\bm v)$ the application of Algorithm~\ref{alg:weighted_abs_sum_forward}
\\[1ex]
\algrow \textbf{If $\min(M,N)=1$:}
\\
\algrow \quad compute $(z_m)_{m=1}^M$ naively
\\
\algrow \textbf{Else:}
\\
\algrow \quad compute $(z_m)_{m=1}^{M/2}=\mathrm{CausalSorting}((s_m)_{m=1}^{M/2},(t_n)_{n=1}^{N/2},(v_n)_{n=1}^{N/2})$
\\
\algrow \quad compute $(z_{m,1})_{m=M/2+1}^{M}=\mathrm{Sorting}((s_m)_{m=M/2+1}^{M},(t_n)_{n=1}^{N/2},(v_n)_{n=1}^{N/2})$
\\
\algrow \quad compute $(z_{m,2})_{m=M/2+1}^{M}=\mathrm{CausalSorting}((s_m)_{m=M/2+1}^{M},(t_n)_{n=N/2+1}^{N},(v_n)_{n=N/2+1}^{N})$
\\
\algrow \quad set $z_m=z_{m,1}+z_{m,2}$ for $m=M/2+1,...,M$
\\[1ex]
\algrow \textbf{return}
$\bm z$
\\
\end{tabularx}
\end{algorithm}

Many current transformers use causal masking. That is, the kernel sum \eqref{eq:kernel_sum} is replaced by \smash{$z_m=\sum_{n=1}^m\Phi(q_m,k_n)v_n$} for $m=1,\ldots,N$. We refer to Figure~\ref{fig:causal_mask} for an illustration of the mask.
In this setting, Algorithm~\ref{alg:weighted_abs_sum_forward} is not directly applicable, since every query now sees a different set of keys.
Nevertheless, we can overcome this issue by a divide-and-conquer algorithm outlined in Algorithm~\ref{alg:causal_sorting}. For simplicity, we assume that the sequence lengths $M$ and $N$ are powers of two.
The main idea of the algorithm (for $M=N$) is that we split the masked kernel matrix into four blocks. Each diagonal block corresponds to a kernel sum with causal masking with sequence length $N/2$. The upper right off-diagonal block is zero (since the mask is zero). The lower left off-diagonal block is a standard kernel sum without masking and can be computed by Algorithm~\ref{alg:weighted_abs_sum_forward} (or the equivalent algorithm for the one dimensional Laplace kernel).

To verify the quasi linear complexity (simply assuming that $M=N=2^K$), we note that for each length $2^k$ for $k=1,\ldots,K$, we call $2^{K-k}$ times Algorithm \ref{alg:weighted_abs_sum_forward} with sequence length $2^k$. In particular, we have for each $k$ the cost $\mathcal O(2^{K-k} 2^k\log(2^k))=\mathcal O(2^{K}\log(2^k))\leq \mathcal O(N\log(N))$. So in total, we obtain the complexity $\mathcal O(KN\log(N))\leq \mathcal O(N\log^2(N))$.
In effect, Algorithm~\ref{alg:causal_sorting} decomposes the triangular mask into $2N-1$ blocks, see the final decomposition in Figure~\ref{fig:causal_mask}, which it traverses serially by the recursion.
This order is not required, however: all blocks can be computed independently, since synchronization is only needed in line~$10$, which can be resolved by atomic additions accumulating into the output tensor.

\subsection{Autoregressive decoding}
During autoregressive decoding, tokens are generated one at a time.
With causal attention, the keys and values of previous tokens do not change and are stored in a KV cache of size $\mathcal O(N(D{+}C))$.
For each new query, softmax attention reads the full cache at cost $\mathcal O(N(D{+}C))$, so decoding is memory-bound rather than compute-bound.
This procedure does not depend on the softmax structure and can be done for any kernel.

\section{Weighted Laplace Sums}
For the Laplace kernel the difficult part is to compute
$\sum_{n=1}^{N}\exp(-|s_m-t_n|)v_n$, where $s_m\in\R$ and $t_n\in\R$.
Splitting at $s_m$ as in \eqref{eq:prefix_sums} gives
$e^{-s_m}\sum_{t_n\le s_m}e^{t_n}v_n+e^{s_m}\sum_{t_n>s_m}e^{-t_n}v_n$,
so the same prefix-sum idea applies.
The individual factors $e^{\pm t_n}$ and $e^{\mp s_m}$, however, overflow in
floating point as soon as the inputs leave a small range, even though their
product is bounded by one.
Algorithm~\ref{alg:weighted_laplace_sum_forward} avoids this by expressing the
prefix sums relative to the current node instead of the origin, which costs
$\mathcal O((N{+}M)(\log_2 N+C))$ as well.

\begin{algorithm}[t]
\caption{Weighted Laplace sum}
\label{alg:weighted_laplace_sum_forward}

\setcounter{algline}{0}
\begin{tabularx}{\linewidth}{@{}r X l l@{}}
\textbf{\#} & \textbf{Step} & \textbf{Work} & \textbf{Memory} \\
\hline

\algrow \textbf{Input}
$\bm s\in\R^M$, $\bm t\in\R^N$ and $\bm v\in\R^{N\times C}$
& -
& $(C{+}1)N{+}M$
\\
\algrow \textbf{Output}
$z_m\coloneqq\sum_{n=1}^{N}\exp(-|s_m-t_n|)v_n\in\R^C\quad m=1,\ldots,M$
& -
& $MC$
\\
\algrow
$\sigma\coloneqq\texttt{argsort}(\bm t)$ s.t. $t_{\sigma(1)}\leq\ldots\leq t_{\sigma(N)}$
& $N\log_2 N$
& $N$
\\
\algrow
$p_m\coloneqq\max\bigl(\{n\colon t_{\sigma(n)}\leq s_m\}\cup\{0\}\bigr)$ for $m=1,\ldots,M$
& $M\log_2 N$
& $M$
\\
\algrow
$\delta_n\coloneqq t_{\sigma(n)}-t_{\sigma(n-1)}\geq 0$ for $n=2,\ldots,N$
& $N$
& $N$
\\
\algrow
$L_n\coloneqq e^{-\delta_n}L_{n-1}+v_{\sigma(n)}\in\R^C$ for $n=1,\ldots,N$, $L_0\coloneqq \bm 0$
& $2NC$
& $(N{+}1)C$
\\
\algrow
$R_n\coloneqq e^{-\delta_{n+1}}R_{n+1}+v_{\sigma(n)}\in\R^C$ for $n=N,\ldots,1$,  $R_{N+1}\coloneqq \bm 0$
& $2NC$
& $(N{+}1)C$
\\
\algrow
$z_m\coloneqq e^{-(s_m-t_{\sigma(p_m)})}L_{p_m}+e^{-(t_{\sigma(p_m+1)}-s_m)}R_{p_m+1}\in\R^C$ for $m=1,\ldots,M$
& $MC$
& $MC$
\\
\algrow \textbf{return}
$\bm z$
&
&
\\
\end{tabularx}
\end{algorithm}

By induction, $L_n=\sum_{j=1}^{n}e^{-(t_{\sigma(n)}-t_{\sigma(j)})}v_{\sigma(j)}$ and
$R_n=\sum_{j=n}^{N}e^{-(t_{\sigma(j)}-t_{\sigma(n)})}v_{\sigma(j)}$, so that with
$t_{\sigma(n)}\leq s_m$ for $n\leq p_m$ and $t_{\sigma(n)}>s_m$ for $n>p_m$ the
last line indeed equals $z_m$, where empty terms are read as zero.

Every exponent occurring in Algorithm~\ref{alg:weighted_laplace_sum_forward} is
non-positive: $\delta_n\geq 0$ by sorting, $s_m-t_{\sigma(p_m)}\geq 0$ and
$t_{\sigma(p_m+1)}-s_m>0$ by the definition of $p_m$.
Hence, all exponentials lie in $(0,1]$, and induction gives
$|L_n^{(c)}|\leq\sum_{j\leq n}|v_{\sigma(j)}^{(c)}|$ and
$|R_n^{(c)}|\leq\sum_{j\geq n}|v_{\sigma(j)}^{(c)}|$ for every channel $c$, so that
$|z_m^{(c)}|\leq\sum_{n=1}^{N}|v_n^{(c)}|$.
All intermediate quantities are thus bounded by the same quantity as the exact
sum itself, independently of the magnitude of $\bm s$ and $\bm t$.

In \texttt{weighted\_laplace\_sum} we provide a parallel implementation following the structure of the \texttt{weighted\_abs\_sum}. Gradients and masking are out of scope.

\section{Vision Transformer}
\label{app:sec:vision_hyper}
We train a small vision transformer \citep{DBKWZUDMHGUH2021ViT} for classification on CIFAR-10: $12$ layers, $3$ heads of dimension $D=64$, hidden dimension $192$, patch size $4$, hence sequence length $64$.
Since the sequence length does not exceed the head dimension, even a $64$-dimensional RKHS suffices to separate all tokens and the capacity mechanism of the previous section cannot bind.
This experiment is therefore deliberately a controlled sanity check: any remaining quality differences must come from properties other than dimension.

Each model is trained once from \emph{scratch} with random initialization, as well as being distilled from a softmax teacher model ($96.57\%$ accuracy) in three stages:
The student shares the architecture of the teacher and differs only in the kernel, so that stage \emph{copy} simply utilizes the teacher weights.
\emph{Layer-by-layer} then finetunes one student attention module at a time to minimize the MSE to the teacher's attention output at that layer, and \emph{end-to-end} finally finetunes all parameters against the teacher logits with temperature-scaled KL.
This distillation setup is identical to the text transformer with the single difference of using the KL divergence instead of the MSE loss for \emph{end-to-end}.

Training from scratch uses AdamW, lr $10^{-3}\to 10^{-5}$ cosine annealing with $100$ warmup epochs, weight decay $0.1$, $1000$ epochs, batch size $128$, RandAugment, MixUp/CutMix and label smoothing $0.1$.
In both the vision and text transformers, the layer-by-layer stage runs $10$k steps per layer. The end-to-end stage runs $50$k steps with KL temperature $T=4$ for the vision transformer, and $10$k steps with the MSE loss for the text transformer (Section~\ref{subsec:num_text}).

\begin{table}[ht]
    \centering
    \caption{CIFAR-10 test accuracy (\%). Training from scratch as well as piecewise distillation.}
    \scriptsize
    \begin{tabular}{llcccc}
    \toprule
    & & & \multicolumn{3}{c}{distillation from softmax} \\
    \cmidrule(lr){4-6}
    Kernel & scale $\tau$ & scratch $\uparrow$ & copy $\uparrow$ & +layerwise $\uparrow$ & +end-to-end $\uparrow$ \\
    \midrule
    \texttt{softmax} (teacher) & $2.828$ & $96.57$ & --- & --- & --- \\
    \texttt{gauss}       & $2.828$ & $96.70$ & $96.10$ & $96.56$ & $96.55$ \\
    \texttt{laplace}     & $6.0$     & $96.71$ & $94.92$ & $96.57$ & $96.64$ \\
    \texttt{riesz}       & $1.0$   & $94.14$ & $58.12$ & $95.61$ & $95.83$ \\ \midrule
    \texttt{add\_riesz}  & $1.0$   & $93.68$ & $55.93$ & $95.84$ & $96.21$ \\
    \texttt{add\_laplace}& $0.5$   & $93.94$ & $43.67$ & $95.96$ & $96.40$ \\
    \texttt{add\_bump}   & $1.5$   & $94.56$ & $43.33$ & $96.13$ & $96.26$ \\
    \texttt{tri}         & $2.828$ & $94.63$ & $39.47$ & $96.01$ & $96.13$ \\
    \texttt{relu}        & $1.0$   & $95.08$ & $62.40$ & $96.12$ & $95.91$ \\
    \texttt{elu}         & $1.0$   & $95.17$ & $43.17$ & $96.17$ & $96.35$ \\
    \texttt{dpfp}        & $1.0$   & $95.15$ & $75.84$ & $96.39$ & $96.41$ \\
    \bottomrule
    \end{tabular}
    \label{tab:vision}
\end{table}
Table~\ref{tab:vision} supports three conclusions.
First, as predicted, no capacity gap appears: all kernels train to within $1.0 \%$ accuracy of softmax after full distillation.
Second, the softmax--Gauss correspondence \eqref{eq:reweighted_gauss} is not merely formal, as the \emph{copy} stage Gauss student already reaches $96.10\%$, so the reweighting argument holds numerically, not just algebraically. Interestingly, \texttt{laplace} with $\tau=6$ also reaches $94.92\%$ without any finetuning, while all other kernels lie at least $19\%$ below.
Third, distillation matches or improves training from scratch for all kernels, in both accuracy and training time; the layer-by-layer stage does most of the work and end-to-end adds less than $1\%$.

\section{Associative recall}
\label{app:sec:associative_retrival}

\citet[Sec.~6.1 \& 6.2]{sis2021FastWeight} introduce associative recall in order to verify the capacity limits.
They start with a set of tokens $\mathcal N\coloneqq\{1,\dots,N\}$ and a learnable embedding $E\colon\mathcal N\to\R^D$. 
Two random permutations $\sigma,\xi$ of $\mathcal N$ are sampled and the sequence of embedded tokens $(u_1,\dots,u_N)$ with $u_n=E(\sigma(n))$ is built. 
To each token $u_n$ a one-hot value \smash{$v_n=e_{\xi(n)}\in\R^N$} is assigned.
The keys are computed via $k_n=W_K[u_n;v_n]$ with a learnable $W_K\in\R^{D\times(D+N)}$.
To recall the value of token $m\in\mathcal N$, a query $q_m=W_QE(m)$ with learnable $W_Q\in\R^{D\times D}$ is built and the value is retrieved as \smash{$\hat v^{(m)}=\sum_{n=1}^N A_{m,n}(\bm q,\bm k)\,v_n\in\R^N$} with $\bm k=(k_n)_{n=1}^N$, $\bm q=(q_m)_{m=1}^N$ and $A$ from \eqref{eq:gram_attention}.
The loss is
\begin{equation}
    \ell({\hat v}^{(m)}, v^{\star(m)})=\frac12\sum_{j=1}^N({\hat v}^{(m)}_j- v^{\star(m)}_j)^2= \frac12 \| {\hat v}^{(m)}- v^{\star(m)}\|_2^2,
\end{equation}where $v^{\star(m)}$ is the value assigned to token $m$, and it is evaluated for all $m\in\mathcal N$. The loss is (up to a factor of $\nicefrac12$) exactly the inner term of \eqref{eq:loss_function}, so that Lemma~\ref{lem:recall_frobenius} also yields the form with the Frobenius norm.
We differ only in the key matrix: we use $W_K\in\R^{D\times D}$ and $k_n=W_Ku_n$, as in the standard transformer.
Our setting is contained in that of \citet{sis2021FastWeight} by extending $W_K$ with a zero matrix $0\in\R^{D\times N}$.

\section{Proofs}
\label{app:sec:proofs}

\begin{proof}[Proof of Theorem~\ref{thm:inf_cap}]
By assumption, there exists some $\mathcal C_0$ function $F$ with $\Phi(q,k)=F(q-k)$.
Now, for arbitrary $N \in \N$, let $(x_n)_{n=1}^N$ be a collection of points with $\varepsilon\coloneqq\min_{n\neq m} \|x_n-x_m\|_2^2>0$.
Then, we have by definition for all $r>0$ with $k_n=q_n=r x_n$ that the Gram matrix $G$ fulfills $G_{n,n}=F(0)$ and $G_{m,n}=F(r(x_m-x_n))\to 0$ as $r\to \infty$ since $r\|x_m-x_n\|\to\infty$.
This yields that $A_{n,n}\to 1$ and $A_{m,n}\to 0$ for all $m\neq n$ such that $\mathrm{Cap}(\Phi)\geq N$ for all $N \in \N$ and thus $\mathrm{Cap}(\Phi)=\infty$.
While many transformers normalize the tokens $x_n$, the keys and queries are obtained via $q_n = W_Q x_n$ and $k_n = W_K x_n$, where $W_Q, W_K \in \mathbb{R}^{D\times D}$ are weight matrices. Choosing $W_Q = W_K = r\,\mathrm{Id}_D$ recovers $q_n = k_n = r x_n$ from the proof.
\end{proof}

\begin{proof}[Proof of Proposition~\ref{prop:softmax_capacity}]
For the Gauss kernel $\Psi(q,k)=\exp(-\frac12\|q-k\|_2^2)$, we have that $\Phi(q,k)=\Psi(q,k)\exp(\frac12(\|q\|_2^2+\|k\|_2^2))$. In particular, if 
$\|k_1\|_2=\cdots=\|k_N\|_2$, the attention matrices $A(\bm q,\bm k)$ with respect to $\Phi$ and $\Psi$ coincide, since
\begin{align}\label{eq:reweighted_gauss}
A_\Phi(\bm q, \bm k) =\frac{\e^{ q_m^\top k_n} }{\sum_{l=1}^N \e^{q_m^\top k_l} }
= \frac{\e^{\frac{1}{2}(\|k_1\|_2^2+\|q_m\|_2^2)} \e^{-\frac{1}{2} \|q_m - k_n\|_2^2} }{\e^{\frac{1}{2}(\|k_1\|_2^2+\|q_m\|_2^2)}\sum_{l=1}^N  \e^{-\frac{1}{2} \|q_m - k_l\|_2^2}}
=A_\Psi(\bm q, \bm k).
\end{align}
Thus, we obtain by the same proof as for Theorem~\ref{thm:inf_cap} that $\mathrm{Cap}(\Phi)=\infty$, provided that we choose $\|x_m\|=1$ for all $m$ in that proof, which is possible because the sphere $\{x\in \R^D\colon \|x\|_2=1\}$ for $D\ge 2$ has infinitely many points.
\end{proof}

\begin{proof}[Proof of Proposition~\ref{prop:stationary_FFM}]
As $\Phi$ is continuous and stationary there is continuous $F\colon \R^D\to \R$ with $\Phi(q,k)=F(q-k)$. 
Since $\Phi$ has a FFM expansion $F(q-k)=\Phi(q,k)=\varphi(q)^\top\varphi(k)$ we know that $F$ is spd and has RKHS $\mathcal H_\Phi$.
By Bochner's theorem, there is a real nonnegative Borel measure $\xi\in \mathcal M$ such that  \smash{$F(x)= \mathcal F[\xi](x)=\int_{\R^D} \e^{-\mathrm{i}\omega^\top x} \d\xi(\omega)$.}
Since $F$ is real symmetric $\xi$, is symmetric too and \smash{$F(x)=\int_{\R^D} \cos(\omega^\top x)\d\xi(\omega)$}.
A feature space of $\mathcal H_\Phi$ is given by $H_0\coloneqq E\oplus O$, where $E$ and $O$ denote the even and odd function classes of $L^2(\xi)$. A feature map of $\Phi$ is $\Phi_0(x)(\omega)\coloneqq [\cos(\omega^\top x), \sin(\omega^\top x)]$, because
\begin{equation}
\langle \Phi_0(q), \Phi_0(k)\rangle _{H_0}
= \int_{\R^D} \cos(\omega^\top (q-k))\d\xi (\omega)= F(q-k)=\Phi(q,k).
\end{equation}
By \citet[Thm.~4.21]{SC2008SupportVectorMachines} the map $V\colon H_0\to \mathcal H_\Phi$ given by $V(h)\coloneqq \langle h, \Phi_0\rangle$ is surjective.
For $h=[h_E,h_O]^\top \in H_0$ with $h_E\in E$ and $h_O\in O$, this operator is given by 
\begin{align}
    V(h)(x)
    &=\langle h, \Phi_0(x)\rangle = \int_{\R^D} h_E(\omega)\cos(\omega^\top x)+h_O\sin(\omega^\top x)\d\xi(\omega)\notag\\&
    =\int_{\R^D} \e^{-\mathrm{i} \omega^\top x} (h_E+\mathrm{i} h_O)\d\xi(\omega)
    =\mathcal F[(h_E+\mathrm{i} h_O)\xi](x).
\end{align}
As the Fourier transformation is injective $V(h)=0$ implies $(h_E+\mathrm{i} h_O)\xi=0$ and thus $h_E\xi=0$ and $h_O\xi=0$. Therefore, $h_E,h_O$ vanish $\xi$-almost everywhere and $V$ is injective. 
Since $\Phi$ has FFM $\varphi\colon \R^D\to \R^{2L}$, it holds $\dim \mathcal H_\Phi\leq 2L <\infty$. Since $V$ is a linear bijection, $H_0$ is finite dimensional, which renders $L^2(\xi)\cong E\oplus O=H_0$ to be finite dimensional. This implies $\operatorname{supp} \xi$ to be a finite set and hence \smash{$\xi=\sum_{l=1}^L \nicefrac{\alpha_l}{2} (\delta_{\omega_l}+\delta_{-\omega_l})$} for $\alpha_l\geq0$ and $\omega_l\in \R^D$.
By Bochner's theorem, we thus obtain
\begin{equation}
\Phi(q,k)=F(q-k)=\int_{\R^D} \cos(\omega^\top (q-k))\d\xi(\omega)=\sum_{l=1}^L \alpha_l \cos(\omega_l^\top (q-k)).
\end{equation}
The function $F$ is a trigonometric polynomial, hence almost periodic.
Its mean value satisfies $M\{F^2\}\ge \frac{1}{2}\sum_{l=1}^L\alpha_l^2>0$ unless $F=0$, by the orthogonality of characters \citep[Prop.~1.5]{Shubin_1978}.
Since $F\in \mathcal C_0$ implies $M\{F^2\}=0$ \citep[Prop.~1.8]{Shubin_1978} we conclude $F=0$.
\end{proof}

\begin{proof}[Proof of Proposition~\ref{prop:spline_is_quasi}]
Any continuous piecewise linear function $f\colon \R \to \R$ with finitely many nodes can be expressed in terms of absolute values \citep[p.~328]{SGRB2014} as
\begin{equation}\label{eq:f_spline_form}
   f(x) = a + bx + \sum_{j=1}^{J} c_j|x - x_j|, \qquad x\in \R.
\end{equation}
Let $s_1,\ldots,s_M\in\R$, $t_1,\ldots,t_N\in\R$ and $v_1,\ldots,v_N\in \R^C$.
Inserting \eqref{eq:f_spline_form} into the kernel sum yields
\begin{equation}
    \sum_{n=1}^N f(s_m-t_n)v_n
    =\underbrace{(a+bs_m)\sum_{n=1}^N v_n - b\sum_{n=1}^N t_nv_n}_{\alpha^{(0)}_m}
    + \sum_{j=1}^J c_j\underbrace{\sum_{n=1}^N |s_m -(x_j+t_n)|v_n}_{\alpha^{(j)}_m}.
\end{equation}
The constant and linear part $(\alpha^{(0)}_m)_{m=1}^M$ can be computed in $\mathcal O(C(N{+}M))$ operations.
For each $j$, we shift the keys to \smash{$t_n^{(j)}\coloneqq t_n+x_j$} and compute \smash{$(\alpha_m^{(j)})_{m=1}^M$} by the fast summation of the distance kernel applied to \smash{$(s_m)_{m=1}^M$}, \smash{$(t_n^{(j)})_{n=1}^N$} and $(v_n)_{n=1}^N$.
Each call costs $\mathcal O(C(N{+}M)\log N)$ operations, which results in a total of $\mathcal O(CJ(N{+}M)\log N)$.
Note that the sorting of $\bm t^{(j)}$ is the same as that of $\bm t$ and does not need to be repeated for each node.
Further, $\bm s$ can be sorted once to avoid the repeated bisection. Therefore, the complexity can be improved to $\mathcal O((N{+}M)(\log(N{+}M)+CJ))$.

Note that the statement extends to the noncontinuous case: every piecewise linear
function with finitely many nodes decomposes into a continuous piecewise linear
function and a piecewise constant function. The latter only requires the evaluation of \smash{$\sum_{n=1}^N \operatorname{sign}(s_m-t_n)v_n$} which can be solved in quasi linear time as explained in \eqref{eq:weighted_sign_sum}.
\end{proof}

\begin{proof}[Proof of Proposition~\ref{prop:riesz_capacity}]
For the Riesz kernel let $N\geq3$ and $s_n,t_n\in \R$ be arbitrary points with $n\in \mathcal N=\{1,\ldots, N\}$.
Further, let $\varepsilon>0$ and $A\coloneqq A_{r_\varepsilon}(\bm s, \bm t)$ denote the attention matrix of $r_\varepsilon$ given $\bm s$ and $\bm t$.
Fix $n_+\in\arg\max_{n\in \mathcal N} t_n$ and
$n_-\in\arg\min_{n\in \mathcal N} t_n$. The map $t\mapsto r_\varepsilon(s_n,t)=|s_n|+|t|-|s_n-t|+\varepsilon$ is monotone for all $n\in \mathcal N$, hence the maximum is attained at $t_{n^\star}$ where $n^\star \in \{n_+,n_-\}$.
For all $n\in \mathcal N\setminus \{n_+,n_-\}$ it holds $A_{n,n^\star}\geq A_{n,n}\geq0$ and thus
\begin{equation}
    \|A_{n,\colon}-e_n\|_2^2 \geq(A_{n,n}-1)^2+A_{n,n^\star}^2 \geq (1-A_{n,n})^2+A_{n,n}^2 \geq \tfrac12 .
\end{equation}
Since $\#(\mathcal N\setminus \{n_+,n_-\})\geq N-2$ we have
$\sum_{n\in \mathcal N}\|A_{n,\colon}-e_n\|_2^2\ge\frac{N-2}{2}$.
For $N=2$ we can choose $s_1=t_1=\tau$ and $s_2=t_2=-\tau$. Then it holds $A\to \mathrm{Id}_2$ as $\tau\to \infty$, so the capacity is exactly $2$.
\end{proof}

\begin{proof}[Proof of Theorem~\ref{thm:additive_capacity}]
To distinguish the Gram and attention matrices with respect to $\Phi$ and $\phi$, we denote them by $G_\Phi$ and $G_\phi$ or $A_\Phi$ and $A_\phi$, respectively. Now let $N\leq\mathrm{Cap}(\phi)$ and $\varepsilon>0$. Then, we know that there exist $\bm s,\bm t\in\R^N$ such that $\|A_\phi(\bm s,\bm t)-\mathrm{Id}_N\|_F^2<\varepsilon$. Define $\bm q,\bm k\in(\R^D)^N$ by $q_{m,d}=s_m$ and $k_{n,d}=t_n$. Then, we obtain by definition that $G_\Phi(\bm q,\bm k)=D G_\phi(\bm s,\bm t)$ and therefore $A_\Phi(\bm q,\bm k)=A_\phi(\bm s,\bm t)$. In particular, we have $\|A_\Phi(\bm q,\bm k)-\mathrm{Id}_N\|_F^2=\|A_\phi(\bm s,\bm t)-\mathrm{Id}_N\|_F^2<\varepsilon$. Since this works for any $\varepsilon>0$, we obtain that $\inf_{\bm q,\bm k}\|A_\Phi(\bm q,\bm k)-\mathrm{Id}_N\|_F^2=0$ and $\mathrm{Cap}(\Phi)\geq N$.

Let $\phi=r_\varepsilon$ be the Riesz kernel.
For arbitrary $D$ and $N=2D$, we can choose $q_n=k_n=\tau e_n$ for $n=1,\ldots, D$ and $q_n=k_n=-\tau e_{n-D}$ for $n=D+1,\ldots, 2D$. Then it holds for $n\neq m$, that $G_{m,n}(\bm q, \bm k)=D\varepsilon$ and for $n=m$, $G_{n,n}(\bm q, \bm k)=D\varepsilon+2\tau$. Hence, we get $A(\bm q, \bm k)\to \mathrm{Id}_N$ as $\tau\to \infty$, so the capacity is at least $2D$.

Symmetry and translation invariance clearly carry over from $\phi$ to $\Phi$.
If $\phi$ is positive definite, then for all $\alpha_1,\ldots,\alpha_N\in \R$ and $x_1,\ldots,x_N\in \R^D$ it holds
\begin{equation}
    \sum_{n,m=1}^N \alpha_n\alpha_m \Phi(x_n,x_m)
    = \sum_{d=1}^D \sum_{n,m=1}^N \alpha_n\alpha_m \phi(x_{n,d},x_{m,d})
    \geq0,
\end{equation}
so $\Phi$ is positive definite.
Now, let $\phi$ be quasi linear and let $q_1,\ldots,q_M\in \R^D$, $k_1,\ldots,k_N\in \R^D$ and $v_1,\ldots,v_N\in \R^C$.
Rearranging the kernel sum yields
\begin{equation}
    z_m
    =\sum_{n=1}^N \Phi(q_m,k_n)v_n
    = \sum_{d=1}^D \sum_{n=1}^N \phi(q_{m,d},k_{n,d})v_n.
\end{equation}
Since $\phi$ is quasi linear, each of the $D$ inner sums can be computed in quasi linear time for all $m=1,\ldots,M$ simultaneously. Summing the results over $d$ costs another $\mathcal O(MDC)$ operations, hence $(z_m)_{m=1}^M$ is computed in quasi linear time for every fixed $D$.
\end{proof}

\begin{proof}[Proof of Lemma~\ref{lem:recall_frobenius}]
Since the normalization in \eqref{eq:gram_attention} sums over all keys, we have
$A_{m, l}(\bm q, \bm k_\sigma)=A_{m,\sigma(l)}(\bm q,\bm k)$. With $A\coloneqq A(\bm q,\bm k)$
this gives $y_{\sigma(j)}=\sum_{l} A_{\sigma(j),\sigma(l)}\,e_{\xi(l)}$, and thus
\begin{align}
\sum_{j=1}^N\|e_{\xi(j)}-y_{\sigma(j)}\|_2^2
&=\sum_{j=1}^N\Big\|\sum_{l=1}^N\big(\delta_{j,l}-A_{\sigma(j),\sigma(l)}\big)e_{\xi(l)}\Big\|_2^2 \notag\\
&=\sum_{j=1}^N\sum_{l=1}^N\big(\delta_{\sigma(j),\sigma(l)}-A_{\sigma(j),\sigma(l)}\big)^2\notag\\
&=\sum_{m=1}^N\sum_{n=1}^N\big(\delta_{m,n}-A_{m,n}\big)^2
=\|A-\mathrm{Id}_N\|_F^2.
\end{align}
The second equality uses that $(e_{\xi(l)})_{l=1}^N$ is an orthonormal basis, and the third
substitutes $n=\sigma(l)$ and $m=\sigma(j)$.
\end{proof}

\end{document}